\documentclass[a4paper, 11pt]{article}
\usepackage[sort&compress,square,comma,authoryear]{natbib}

\usepackage{a4wide}
\usepackage{algorithm}
\usepackage{algorithmic}
\usepackage[algo2e,linesnumbered,ruled]{algorithm2e}
\usepackage{afterpage}

\usepackage{amsfonts}
\usepackage{amsmath}
\usepackage{amssymb}
\usepackage{amstext}
\usepackage{amsthm}
\usepackage{caption}
\usepackage{dsfont}
\usepackage{subcaption}
\usepackage{microtype}      
\usepackage{booktabs}
\usepackage{dirtytalk}
\usepackage{enumitem}
\usepackage{mathtools}
\usepackage{booktabs}
\usepackage[T1]{fontenc}
\usepackage{graphicx}
\usepackage{pifont}
\usepackage{setspace}
\usepackage{hyperref}
\usepackage[capitalize]{cleveref}

\usepackage{subcaption}
\usepackage{graphicx}

\usepackage{wrapfig}

\usepackage{thmtools}
\usepackage{thm-restate}

\usepackage{nicefrac}
\usepackage[dvipsnames]{xcolor} 
\usepackage[parfill]{parskip}

\definecolor{salmon}{RGB}{234,153,153}
\definecolor{cornflowerblue}{RGB}{6,69,173}
\hypersetup{
    colorlinks,
    linkcolor={cornflowerblue},
    citecolor={cornflowerblue},
    urlcolor={salmon}
}

\usepackage{tabularx}
\newcolumntype{Y}{>{\centering\arraybackslash}X}
\newcolumntype{Z}{>{\raggedleft\arraybackslash}X}
\usepackage{tikz}
\usetikzlibrary{calc, fadings, decorations, shapes, positioning, arrows}

\theoremstyle{plain}

\newtheorem{theorem}{Theorem}[section]

\newtheorem{lemma}[theorem]{Lemma}

\theoremstyle{definition}
\newtheorem{definition}[theorem]{Definition}
\newtheorem{assumption}[theorem]{Assumption}
\theoremstyle{remark}

\newtheorem{condition}[theorem]{Condition}

\AfterEndEnvironment{theorem}{\noindent\ignorespaces}
\AfterEndEnvironment{proposition}{\noindent\ignorespaces}
\AfterEndEnvironment{lemma}{\noindent\ignorespaces}
\AfterEndEnvironment{corollary}{\noindent\ignorespaces}
\AfterEndEnvironment{definition}{\noindent\ignorespaces}
\AfterEndEnvironment{remark}{\noindent\ignorespaces}

\DeclarePairedDelimiter\abs{\lvert}{\rvert}
\DeclarePairedDelimiter\norm{\lVert}{\rVert}
\DeclareMathOperator*{\argmax}{arg\,max}
\DeclareMathOperator*{\argmin}{arg\,min}

\newcommand{\realNumber}{\mathbb{R}}
\newcommand{\naturalNumber}{\mathbb{N}}

\newcommand{\probP}{\mathds{P}}
\newcommand{\expectE}{\mathds{E}}

\newcommand{\indep}{\perp \!\!\! \perp}

\DeclarePairedDelimiterX{\expectarg}[1]{[}{]}{%
  \ifnum\currentgrouptype=16 \else\begingroup\fi
  \activatebar#1
  \ifnum\currentgrouptype=16 \else\endgroup\fi
}

\newcommand{\testD}{\probP_{\textrm{test}}}
\newcommand{\dataset}{\mathcal{D}}
\newcommand{\orthoM}{\psi}

\newcommand{\actions}{\mathcal{A}}

\newcommand{\policy}{\pi}
\newcommand{\Value}{V}

\let\oldding\ding
\renewcommand{\ding}[2][1]{\scalebox{#1}{\oldding{#2}}}

\title{\textbf{Learning Decision Policies with Instrumental Variables\\
through Double Machine Learning}}

\date{}

\usepackage{authblk}

\author[1]{Daqian Shao}
\author[2]{Ashkan Soleymani}
\author[1]{Francesco Quinzan}
\author[1]{Marta Kwiatkowska}

\affil[1]{Department of Computer Science, University of Oxford}
\affil[2]{Department of Electrical Engineering and Computer Science, Massachusetts Institute of Technology}

\begin{document}

\maketitle

\begin{abstract}
A common issue in learning decision-making policies in data-rich settings is spurious correlations in the offline dataset, which can be caused by \textit{hidden confounders}. Instrumental variable (IV) regression, which utilises a key unconfounded variable known as the \textit{instrument}, is a standard technique for learning causal relationships between confounded \textit{action}, \textit{outcome}, and \textit{context} variables. Most recent IV regression algorithms use a two-stage approach, where a deep neural network (DNN) estimator learnt in the first stage is directly plugged into the second stage, in which another DNN is used to estimate the causal effect. Naively plugging the estimator can cause heavy bias in the second stage, especially when regularisation bias is present in the first stage estimator. We propose DML-IV, a non-linear IV regression method that reduces the bias in two-stage IV regressions and effectively learns high-performing policies. We derive a novel learning objective to reduce bias and design the DML-IV algorithm following the \textit{double/debiased machine learning (DML)} framework. The learnt DML-IV estimator has strong convergence rate and $O(N^{-1/2})$ suboptimality guarantees that match those when the dataset is unconfounded.
DML-IV outperforms state-of-the-art IV regression methods on IV regression benchmarks and learns high-performing policies in the presence of instruments.
\end{abstract}

\section{Introduction}

Recent advances in deep learning (DL) have greatly facilitated the learning of decision-making policies in data-rich settings, but they often lack optimality guarantees.
A common issue for learning from offline observational data is the existence of spurious correlations, which are relationships between variables that appear to be causal, but in fact are not. For example, suppose we have aeroplane ticket sales and pricing data in a ticket demand scenario~\cite{Hartford2017DeepPrediction}, and we wish to learn a policy from this offline data that maximises revenue. During holiday season, observational data may contain evidence of a concurrent surge in both ticket sales and prices, 
which may result in the learning algorithm to learn 
an incorrect policy that higher ticket prices will drive higher sales.

Spurious correlations are often caused by \textit{hidden confounders}~\cite{Pearl2000causality}, which are unobserved variables that influence both the \textit{actions} (or \textit{interventions}) and the \textit{outcome}. In the aeroplane ticket example, the occurrence of popular events and holidays serves as a hidden confounder that raises both ticket prices (actions) and sales (outcome). To properly account for these hidden confounders and understand the true causal effect of actions, we need to model the causal (or structural) relationship between the action and the outcome, which is expressed through a \textit{causal function}. However, learning the causal function in the presence of hidden confounders is known to be challenging and sometimes infeasible~\cite{Shpitser2008CompleteHierarchy}.

A popular approach to deal with hidden confounders is via \emph{instrumental variables} (IVs)~\cite{wright1928}, which are heterogeneous 
random variables that only affect the action, but not the outcome. These IVs have been used extensively to identify the causal effect of actions in many applications, including econometrics~\cite{Reiersol1945,Angrist2009}, drug testings~\cite{Angrist1996}, and social sciences~\cite{Angrist1990}. In the aeroplane ticket example, we can employ supply cost-shifters (e.g., fuel price) as instrumental variables, as their variations are
independent of the demand for aeroplane tickets and affect sales solely via ticket prices~\cite{Blundell2012}.

We focus on the problem of learning the causal function in the presence of hidden confounders using IVs (known as \textit{IV regression}), in order to learn a decision policy that maximises the expected outcome in this setting (which we refer to as the \textit{offline IV bandit} problem, described in~\cref{sec:offlineivbandit}) and comes with suboptimality guarantees. Two-stage least squares (2SLS)~\cite{Angrist1996} is a classical IV regression algorithm, which has been extended to non-linear settings that utilise machine learning (ML) techniques, including deep neural networks (DNNs), to learn the causal function.
The use of DNNs allows for greater flexibility in IV regression, as it does not impose strong assumptions on the functional form and can learn directly from data. However, regularisation is often employed to trade-off overfitting with the induced regularisation bias, especially for high-dimensional inputs. Both regularisation bias and overfitting may cause heavy bias~\cite{Chernozhukov2018Double/debiasedParameters} in estimating the causal function when the first stage estimator is naively plugged in, which causes slow convergence of the causal function estimator.

\textit{Double/Debiased Machine Learning}~\cite{Chernozhukov2018Double/debiasedParameters} (DML) is a statistical technique that provides an unbiased estimator with convergence rate guarantees for general two-stage regressions. DML relies on having a Neyman orthogonal~\cite{Neyman1965} score function to deal with regularisation bias, and uses cross-fitting, that is, an efficient form of (randomised) data splitting, to tackle overfitting bias. However, the use of DML for IV regression that utilises neural networks has not been explored.

In this work, we propose DML-IV, a novel IV regression algorithm that adopts the DML framework to provide an unbiased estimation of the causal function with fast convergence rate guarantees. We derive a novel Neyman orthogonal score for IV regression, and design a cross-fitting regime such that, under mild regularity conditions, our estimator is guaranteed to converge at the rate of $N^{-1/2}$, where $N$ is the sample size. We then extend DML-IV to solve the offline IV bandit problem, where we derive a policy from the DML-IV estimator and provide a $O(N^{-1/2})$ suboptimality bound with high probability that matches the suboptimality bounds of \textit{unconfounded} offline bandit algorithms~\cite{Jin2021,Nguyen-Tang2022}. Finally, we evaluate DML-IV on multiple benchmarks for IV regression and offline IV bandits, where superior results are demonstrated compared to state-of-the-art (SOTA) methods.

\vspace{+0.5em}
\noindent\textbf{Novel Contributions.}
\vspace{-0.5em}
\begin{itemize}
\itemsep-0.2em
\item We propose DML-IV, 
a novel IV regression algorithm that leverages the DML framework to provide unbiased estimation of the causal function.
\item We derive a novel, Neyman orthogonal, score function for IV regression, 
and design a cross-fitting regime for the DML-IV estimator to mitigate the bias.
\item We provide the first convergence rate guarantees for IV regression algorithms that use DL. Namely, we show that DML-IV converges at $N^{-1/2}$ rate leading to $O(N^{-1/2})$ suboptimality for the derived policy.
\item 
On a range of IV regression and offline IV bandit benchmarks, including two real-world datasets, we experimentally demonstrate that DML-IV outperforms other SOTA methods.
\end{itemize}

\subsection{Related Works}

\noindent\textbf{IV Regression.}
A number of approaches have been developed to extend the two-stage least squares (2SLS) algorithm~\cite{Angrist1996} to non-linear settings. A common approach is to use non-linear basis functions, such as Sieve IV~\cite{Newey2003,Blundell2007,Chen2018}, Kernel IV~\cite{Singh2019} and Dual IV~\cite{Muandet2020}. 
These methods enjoy theoretical benefits, but their flexibility is limited by the set of basis functions. 
More recently, DFIV~\cite{Xu2020} proposed to use basis functions parameterised by DNNs, which remove the restrictions on the functional form. Another approach is to perform stage 1 regression through conditional density estimation~\cite{Darolles2011}, where DeepIV~\cite{Hartford2017DeepPrediction} adopts DNNs to perform these regressions. DeepGMM~\cite{Bennett2019DeepAnalysis} is a DNN-based method that is inspired by the Generalised Method of Moments (GMM) to find a causal function that ensures the regression residual and the instrument are independent. The learning procedure of DeepGMM does not offer stability comparable to 2SLS approaches, as it is based on solving a smooth zero-sum game, similar to training Generative Adversarial Networks~\cite{Goodfellow2014}.
Our approach allows DNNs in both stages and compares favourably to Deep IV, DeepGMM, Kernel IV and DFIV.

\noindent\textbf{Double Machine Learning (DML).}
DML was originally proposed for semiparametric regression~\cite{Robinson1988}; it relies on the derivation of a score function, which describes the regression problem that is Neyman orthogonal~\cite{Neyman1965}. DML was later extended by adopting DNNs for generalised linear regressions~\cite{Chernozhukov2021AutomaticRegression}. Its strength is that it provides unbiased estimations for causal effects when the causal effect is identifiable~\cite{Jung2021} or there are no hidden confounders~\cite{Chernozhukov2022RieszNetForests}. DML offers strong ($N^{-1/2}$, where $N$ is the size of the dataset) guarantees on the convergence rate, even in the presence of high-dimensional input.

There are previous works on combining DML with IV regression, but they are mainly focused on linear and partially linear models. \citet{belloni2012sparse} propose a method to use Lasso and then Post-Lasso methods for the first stage estimation of linear IV to estimate the optimal instruments. To avoid selection biases, \citet{belloni2012sparse} leverage techniques from weak identification robust inference. In addition, \citet{chernozhukov2015post} propose a Neyman-orthogonalised score for the linear IV problem with control and instrument selection, to potentially be robust to regularisation and selection biases of Lasso as a model selection method. Neyman orthogonality for partially linear models with instruments was primarily discussed in the 
work of \citet{Chernozhukov2018Double/debiasedParameters}. Furthermore, DML techniques for identifying the local average treatment effects (LATE) for nonlinear models with a binary instrument and treatment (action) have been explored before \citep{chernozhukov2024applied}. For additional discussion, we refer to the book \citep{chernozhukov2024applied}.

DML for semiparametric models~\cite{Chernozhukov2022LocallyEstimation,Ichimura2022} has been previously applied to solve the nonparametric IV (NPIV) problem. However, their methods require that the average moment of the Neyman orthogonal score is affine (linear) in the nuisance parameters. Therefore, when applied to solve NPIV, functional assumptions regarding the IV set and the residual function were made. Such assumptions are not required in our work since we are considering a different problem setting and their Neyman orthogonal score is very different from ours. To the best of our knowledge, there is no work that adopts the DML framework for IV regression with DNNs.

\noindent\textbf{Causality.}
Doubly robust estimation for causality problems predominantly revolved around the estimation of average treatment effects (ATE)~\citep{robins1994estimation, funk2011doubly,benkeser2017doubly,bang2005doubly,sloczynski2018general}. Recently, there has been a surge in doubly robust identification of causal structures beyond the ATE settings. \citet{korth,quinzan2023drcfs} focus on finding direct causes of the target variable by orthogonalised scores. \citet{angelis2023doubly} extend this line for testing Granger causality in the time-series domain. In this work, we focus on doubly robust estimation of the counterfactual prediction function, a central problem in the field of causal inference, which could be of independent interest beyond the IV settings.

\noindent\textbf{Offline Bandit.}
Most bandit algorithms assume unconfoundedness (e.g., \citet{Nguyen-Tang2022,Subramanian2022CausalInterventions}). 
For bandit algorithms that consider hidden confounders, most of them work in the online setting, aiming to learn the best policy from scratch using the least amount of online interactions~\cite{Zhang2020DesigningApproach,Subramanian2022CausalInterventions}, or with the help of a pre-collected dataset~\cite{Lu2020RegretKnowledge}. Few works are dedicated to the offline confounded bandit, where only the offline dataset is provided, as it is essentially a causal inference problem. However, offline reinforcement learning (RL) with hidden confounders has been studied. \citet{Pace2023} develop a pessimistic algorithm based on the Delphic uncertainty
due to hidden confounders, while other methods adopt IV regression in combination with value iteration~\cite{Liao2021InstrumentalLearning} and Actor-Critic methods~\cite{Li2021} to learn policies in offline RL. Offline policy evaluation (OPE) under hidden confounders has also been studied. Using IVs, doubly robust estimators for policy values are derived through efficient influence functions~\cite{Xu2023} and marginalised importance sampling~\cite{Fu2022OfflineProcesses}. \citet{Bennett2021} solve OPE under an infinite-horizon ergodic MDP with hidden confounders using states and actions as proxies for the hidden confounders to identify policy values. \citet{Chen2021} consider the OPE problem in a standard unconfounded MDP, where they view the previous (action, state) pair as the instrument for the Bellman residual estimation problem of the current (action, state) pair and directly apply existing IV regression methods to estimate the Q value. We consider the setting of the offline confounded bandit with IVs, for which we leverage DML to obtain convergence and suboptimality guarantees.

\section{Preliminaries}
\subsection{Notation}
We use uppercase letters such as $C$ to denote random variables.
An observed realisation of $C$ is denoted by a lowercase letter $c$. We abbreviate $\expectE[R \lvert C=c]$, a realisation of the conditional expectation $\expectE[R \lvert C]$, as $\expectE[R \lvert c]$. $[N]$ denotes the set $\{1,...,N\}$ for $N\in\naturalNumber$. We write $\expectE[R\lvert do(A=a)]$ for the expectation of $R$ under \emph{do}  intervention~\cite{Pearl2000causality} of setting $A=a$. We use $\norm{\cdot}_p$ to denote the functional norm, defined as $\norm{f}_p\coloneqq\expectE[\abs{f(C)}^p]^{1/p}$, where the measure is implicit from the context. For a function $f$, we use $f_0$ to denote the true function and $\hat{f}$ an estimator of the true function. We use $O$ and $o$ to denote big-O and little-o notations~\cite{Weisstein2023} respectively.


\subsection{Contextual IV Setting}
We begin with a description of the contextual IV setting~\cite{Hartford2017DeepPrediction} that we use in this paper. We observe an \textit{action} $A\in\mathcal{A}\subseteq\realNumber^{d_A}$, a \textit{context} $C\in\mathcal{C}\subseteq\realNumber^{d_C}$, an \textit{instrumental variable (IV)}  $Z\in\mathcal{Z}\subseteq\realNumber^{d_Z}$ and an \textit{outcome} $R\in\realNumber$, where there exist \textit{unobserved confounders} that affect all of $A$, $C$ and $R$ through a hidden variable (or \textit{noise}) $\epsilon$. IV directly affects the action $A$, does not directly affect the outcome $R$ and is not correlated with the hidden confounder $\epsilon$. These causal relationships are illustrated in~\cref{fig:SCM} and are represented by the following structural causal model~\cite{Pearl2000causality}:
\begin{align}
    &R\coloneqq f_r(C,A)+\epsilon,\label{eq:reward}\quad\expectE[\epsilon]=0, \quad\expectE[\epsilon\lvert A,C]\neq0,
\end{align}
where $f_r$ is an unknown, continuous, and potentially non-linear causal function, and $\expectE[\epsilon\lvert A,C]$ is not necessarily zero. Denote the set of observations $(c_i,z_i,a_i,r_i)$, where $i\in[N]$, generated from this model as the \textit{offline dataset} $\dataset$. The goal of this paper is to learn the \textit{counterfactual prediction function}~\cite{Hartford2017DeepPrediction},
\begin{align*}
h_0(C,A):=f_r(C,A)+\expectE[\epsilon \lvert C]=\expectE[R\lvert do(A), C],
\end{align*}
which is the expected outcome under $do(A)$ intervention conditional on $C$, from the offline dataset $\dataset$. This task is also known as \textit{IV regression}, and we aim to estimate $h_0$ using a DNN. The term $\expectE[\epsilon \lvert C]$ is typically nonzero\footnote{In the setting where $\expectE[\epsilon\lvert C]=0$ is assumed~\cite{Bennett2019DeepAnalysis,Xu2020}, $h_0=f_r$ and all our results apply.}, but learning $h_0$ still allows us to compare between different actions when given a context as $h_0(C,a_1)-h_0(C,a_2)=f_r(C,a_1)-f_r(C,a_2)$ for all $a_1,a_2\in\actions$, and in particular, $\argmax_a h_0(C,a)=\argmax_a f_r(C,a)$.

Generally, $h_0$ is allowed to be infinite-dimensional, as commonly seen in nonparametric IV literature~\cite{Newey2003}. We also allow $h_0$ to be infinite-dimensional for the Neyman orthogonal score introduced in~\cref{sec:neyman}, but later, in~\cref{sec:dmliv}, we  restrict $h_0$ to be finite-dimensional and parameterised  to obtain the theoretical results of the convergence rate and the suboptimality bound of $O(N^{-1/2})$.

The challenge of learning $h_0$ from $\dataset$ is that $\expectE[\epsilon\lvert C,A]\neq 0$, which reflects the existence of hidden confounders that obscure the true causal effect. It has been shown~\cite{Bareinboim2012} that we cannot learn the causal effect of actions in the presence of hidden confounders without structural assumptions. Fortunately, IVs enable the identification of $h_0$ if the following
assumptions hold:
\begin{assumption}
(a) $\epsilon$ is additive to $R$ and $\expectE[\epsilon]=0$; (b) $Z \indep \epsilon \mid C$; and (c) $\probP(A\lvert C,Z)$ is not constant in $Z$.\label{assump:struc}
\end{assumption}
Intuitively,~\cref{assump:struc} (a) and (b), introduced by~\citet{Newey2003}, is known as the exclusion restriction, and requires that the instrument $Z$ is uncorrelated with the hidden confounder $\epsilon$. \cref{assump:struc} (c), known as the relevance condition, ensures that $Z$ induces variation in action and should be satisfied by the data generation policy. These assumptions are standard for the IV setting~\cite{Newey2003,Xu2020,Singh2019}, and allow for the minimal condition to identify the causal effect.

\begin{figure}[tb]
\centering
\includegraphics[width=0.5\textwidth]{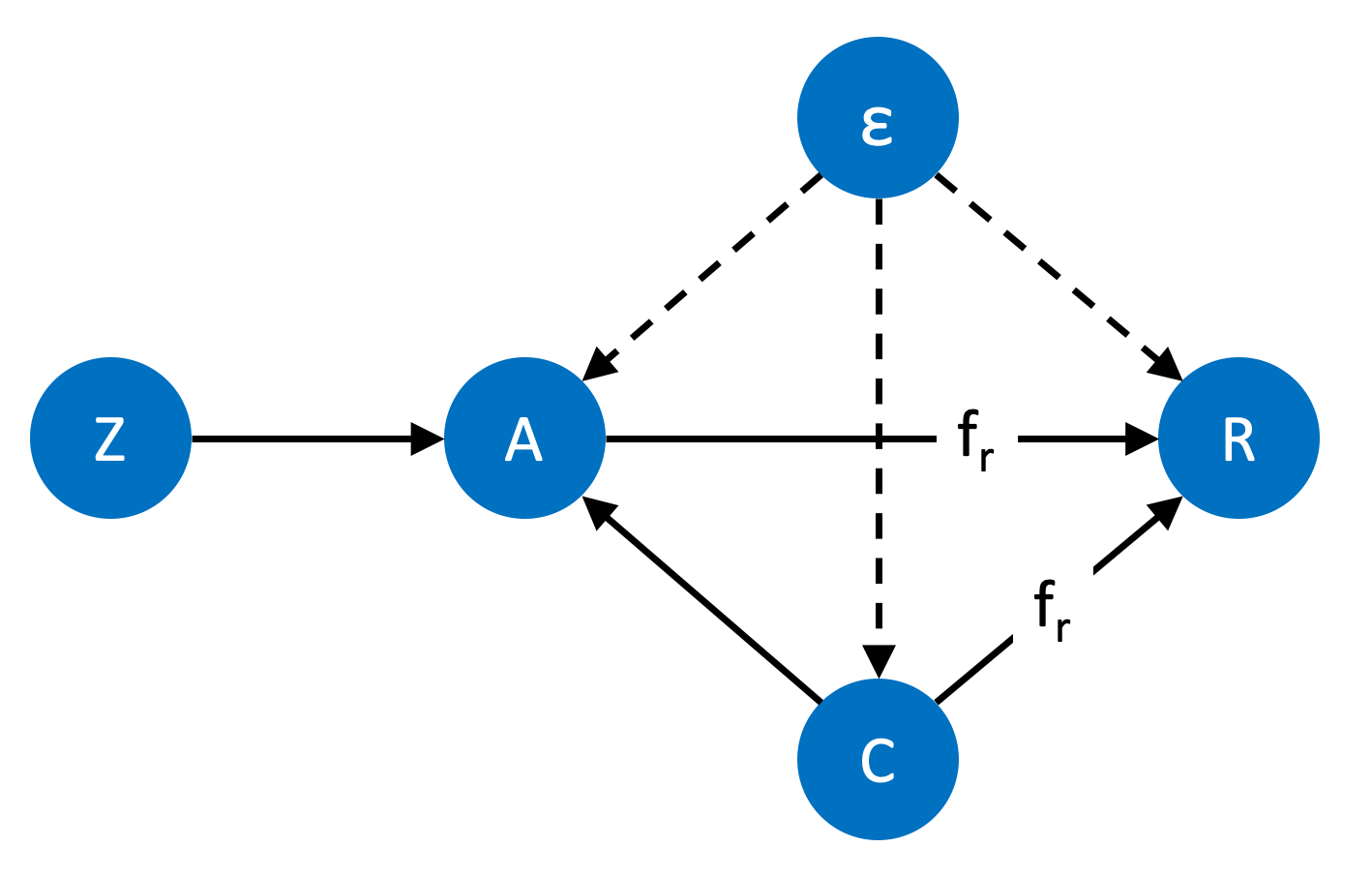}
    \caption{The causal graph of the contextual IV setting, where $R=f_r(C,A)+\epsilon$ and $Z$ is an \textit{instrumental variable} that affects $R$ only through $A$.}
    \label{fig:SCM}
\end{figure}

\subsection{Offline IV Bandit}\label{sec:offlineivbandit}

The learnt estimator of $h_0$ from the offline dataset $\dataset$ can be used to solve the \textit{offline bandit problem in the contextual IV setting}~\cite{Zhang2022}, that is, to identify a (deterministic) policy $\policy:\mathcal{C}\rightarrow\actions$ that maximises the value $\Value(\pi)\coloneqq\expectE_{c\sim\testD}[R\lvert do(A=\pi(c)),c]=\expectE_{c\sim\testD}[h_0(c,\pi(c))]$, which is the expected outcome when performing actions following $\policy$. $\testD$ is a test context distribution that can potentially differ from the distribution of $\dataset$. The optimal policy $\policy^*$ should satisfy $\Value(\policy^*)=\max_\policy \Value(\policy)$, and suboptimality is defined as $\textrm{subopt}(\policy)\coloneqq\Value(\policy^*)-\Value(\policy)$. We see that the optimal policy $\policy^*$ can be retrieved from $h_0$ by selecting $\policy^*(c)=\argmax_{a\in\mathcal{A}} h_0(c,a)$.


\subsection{Two-Stage IV Regression}
In order to identify $h_0$, a key observation~\cite{Newey2003} is that,  by taking the expectation on both sides of~\cref{eq:reward} conditional on $(C,Z)$, we have
\begin{align}
\expectE[R\lvert C,Z]&=\expectE\Big[f_r(C,A)+\expectE[\epsilon\lvert C]\Big\lvert C,Z\Big]\nonumber\\
&=\expectE[h_0(C,A)\lvert C,Z]\label{eq:h_exp}\\
&=\int h_0(C,A) \probP(A\lvert C,Z) dA,\nonumber
\end{align}
where the expectation $\expectE[R\lvert C,Z]$ and the distribution $\probP(A\lvert C,Z)$ are both observable. However, solving this equation analytically is ill-posed~\cite{Nashed1974}. This is an inverse problem for definite integrals that requires the derivation of a function inside the definite integral based on numerical integration values, which is thus not solvable analytically. Recent IV regression methods instead estimate $\hat{h}$ in some space of continuous functions $\mathcal{H}$ by solving the following optimisation problem with a two-stage approach:
\begin{equation}
    \min_{h\in\mathcal{H}} \expectE[(R-\expectE[h(C,A)\lvert C,Z])^2]\label{eq:optim}.
\end{equation}
In the first stage, the conditional expectation $\expectE[h(C,A)\lvert c,z]$ is learnt as a function of $(c,z)$ using observations, and in the second stage, the loss in~\cref{eq:optim} is minimised using the estimator obtained in stage 1. In both stages, linear regression or parametric ML methods, such as DNN, can be used to learn the true functions.

\subsection{Double Machine Learning}\label{sec:dml}

DML is a parameter estimation method that can mitigate certain biases in the learning process \citep{Chernozhukov2018Double/debiasedParameters,Chernozhukov2021AutomaticRegression,Chernozhukov2022RieszNetForests}, which has been extended to work with ML methods, including DL. DML considers the problem of estimating a function of interest $h$ as a solution to an equation of the form 
\begin{equation}
\label{eq:score_function}
    \expectE[\orthoM(\dataset;h, \eta) ] =0,
\end{equation}
where $\orthoM$ is referred to as a score function. Here, $\eta$ is a nuisance parameter, which is of no direct interest, but must be estimated to obtain $h$. DML provides a set of tools to derive an unbiased estimator of $h$ with convergence rate guarantees, even when the nuisance parameter $\eta$ suffers from regularisation, overfitting and other type of biases present in the training of ML models, which typically causes slow convergence when learning $h$.

In order to estimate $h$, DML reduces biases by using score functions $\orthoM$ that are Neyman orthogonal~\cite{Neyman1965} in $\eta$, which require the Gateaux derivative
\begin{align}
\label{eq:neyman}
\frac{\partial}{\partial r}\Big\lvert_{r=0} \expectE[\orthoM(\dataset;h_0,\eta_0+ r\eta)] = 0,
\end{align}
for all $\eta$. Here, $h_0$ and $\eta_0$ are the true parameters that minimise the expected score, that is, $\expectE[\orthoM(\dataset;h_0,\eta_0)] = 0$. Intuitively, the condition in~\cref{eq:neyman} is met if small changes of the nuisance parameter do not significantly affect the score function around the true parameter $h_0$. Neyman orthogonality is key in DML, as it allows fast convergence guarantees for $h$, even if the estimator for the nuisance parameter $\eta$ is biased. For score functions that are Neyman orthogonal, we define DML with \textit{K-fold cross-fitting} as follows.
\begin{definition}[DML, Definition 3.2~\citep{Chernozhukov2018Double/debiasedParameters}]\label{defn:dml}
Given a dataset $\mathcal{D}$ of $N$ observations, consider a score function $\orthoM$ as in~\cref{eq:score_function}, and suppose that $\orthoM$ is Neyman orthogonal that satisfies~\cref{eq:neyman}. Take a \textit{K-fold} random partition $\{I_k\}^K_{k=1}$ of observation indices $[N]$ each with size $n=N/K$, and let $\mathcal{D}_{I_k}$ be the set of observations $\{\mathcal{D}_i:i\in I_k\}$. Furthermore, define $I^c_k\coloneqq [N]\setminus I_k$ for each fold $k$, and construct estimators $\hat{\eta}_k$ of the nuisance parameter using $\mathcal{D}_{I^c_k}$. Then, construct an estimator $\hat{h}$ as a solution to the equation
\begin{align}
\label{eq:sol_model}
    \frac{1}{K}\sum_{k=1}^K \hat{\expectE}_{k}[\orthoM(\dataset_{I_k};\hat{h},\hat{\eta}_k)]=0,
\end{align}
where $\hat{\expectE}_k$ is the empirical expectation over $\mathcal{D}_{I_k}$.
\end{definition}
In the definition above, $\hat{h}$ is defined as a solution to~\cref{eq:sol_model}. In practice, however, finding an exact solution may not be feasible. To circumvent this problem, we can also define the estimator of interest $\hat{h}$ as an $\epsilon_N$-approximate solution to~\cref{eq:sol_model}, where $\epsilon_N=O(N^{-1/2})$, which allows for a small optimisation error.

\section{DML-IV Algorithm}\label{sec:3}
We now present the main contributions of this paper.
The key to our results is the DML-IV algorithm, a novel two-stage IV regression algorithm utilising DNNs in both stages that provides guarantees on the convergence rate by leveraging the DML framework (see~\cref{sec:dml}).  
The DML-IV estimator is then utilised to solve an offline IV bandit (see~\cref{sec:offlineivbandit}) by retrieving a deterministic policy with suboptimality guarantees that match those of the uncounfounded bandit.

Firstly, we remark that, in order to
estimate the counterfactual prediction function $h_0$ with convergence rate guarantees, we need a Neyman orthogonal score.
We let $g_0(h,c,z)\coloneqq\expectE[h(C,A)\lvert c,z]$ and let $\mathcal{G}$ to be some function space that includes $g_0$ and its potential estimators $\hat{g}$. Unfortunately, the standard score (or loss) function for two-stage IV regression $\ell=(R-g(h,c,z))^2$ in~\cref{eq:optim} is not Neyman orthogonal (details in~\cref{appen:score}), which means that small misspecifications or bias on $g$ may lead to significant changes to this loss function, and there are no guarantees on the convergence rate if the first stage estimator $\hat{g}$ is naively plugged into the loss to estimate $h_0$. 
To address this, we first derive a novel Neyman orthogonal score function for the IV regression problem and then design a DML algorithm with K-fold cross-fitting adapted to the IV regression problem.


\subsection{Neyman Orthogonal Score}\label{sec:neyman}
We first derive a novel Neyman orthogonal score for learning $h_0$ in the contextual IV setting. The key to constructing a Neyman orthogonal score usually involves estimating additional nuisance parameters~\cite{Chernozhukov2018Double/debiasedParameters} and adding terms to the original score function to debias it, so we first select relevant quantities that should be estimated as nuisance parameters. Following two-stage IV regression approaches~\cite{Hartford2017DeepPrediction}, estimating $g_0$ is essential for identifying $h_0$, so we will estimate it as a nuisance parameter. We found that, by additionally estimating $s_0(c,z)\coloneqq\expectE[R\lvert c,z]$ inside some function space $\mathcal{S}$, we can construct a new score function
\begin{equation}
\orthoM(\dataset;h,(s,g))=(s(c,z)-g(h,c,z))^2,\label{eq:neyman_score}
\end{equation}
by replacing $R$ in the standard score with $s(c,z)$. Here, the nuisance parameters are $\eta=(s,g)$. We see that $\orthoM$ is a valid score function since $\expectE[\orthoM(\dataset;h_0,(s_0,g_0))]=0$ with the true functions $(s_0,g_0)$ by~\cref{eq:h_exp}, and the next theorem shows that our score function is in fact Neyman orthogonal by checking its Gateaux derivative vanishes at $(h_0,(s_0,g_0))$, where the proof is deferred to~\cref{appen:neyman}.
\begin{theorem}\label{thm:neyman}
The score function $\orthoM(\dataset;h,(s,g))=(s(c,z)-g(h,c,z))^2$ obeys the Neyman orthogonality conditions at $(h_0,(s_0,g_0))$.
\end{theorem}
This Neyman orthogonal score function is abstract, in the sense that it allows for general estimation methods for $g_0$ and $s_0$, as long as they satisfy certain convergence conditions, which are introduced in the next section.

\subsection{Learning Causal Effects through DML}\label{sec:dmliv}

\begin{algorithm}[tb]
   \caption{DML-IV with K-fold cross-fitting}
   \label{alg:dml-iv-kf}
\begin{algorithmic}
   \STATE {\bfseries Input:} Dataset $\dataset$ of size $N$, number of folds $K$ for cross-fitting, mini-batch size $n_b$
    \STATE {\bfseries Output:} The DML-IV estimator $h_{\hat{\theta}}$
   \STATE Get a partition $(I_k)^K_{k=1}$ of dataset indices $[N]$
   \FOR{$k=1$ {\bfseries to} $K$}
   \STATE $I^c_k\coloneqq[N]\setminus I_k$
   \STATE Learn $\hat{s}_k$ and $\hat{g}_k$ using $\{(\dataset_i):{i\in I^c_k}\}$
   \ENDFOR
   \STATE Initialise $h_{\hat{\theta}}$
   \REPEAT
   \FOR{$k=1$ {\bfseries to} $K$}
   \STATE Sample $n_b$ data $(c_i^k,z_i^k)$ from $\{(\dataset_i):{i\in I_k}\}$
   \STATE $\mathcal{L}=\hat{\expectE}_{(c_i^k,z_i^k)}\left[(\hat{s}_k(c,z)-\hat{g}_k(h_\theta,c,z))^2\right]$
   \STATE Update $\hat{\theta}$ to minimise loss $\mathcal{L}$ 
   \ENDFOR
    \UNTIL{convergence}
\end{algorithmic}
\end{algorithm}

With the Neyman orthogonal score, we now introduce DML-IV. While the DML-IV algorithm does not require any assumptions on $h$, we assume that $h$ is finite-dimensional and parameterised for the theoretical analysis of DML-IV. Let $h_0=h_{\theta_0}$ and $\Theta\subseteq \realNumber^{d_\theta}$ be a compact space of parameters of $h$, where the true parameter $\theta_0\in\Theta$ is in the interior of $\Theta$, and $\mathcal{H}\coloneqq\{h_\theta:\theta\in\Theta\}$ is the function space of $h$. The procedure of the DML-IV algorithm for estimating $h_0$ is described in~\cref{alg:dml-iv-kf}. Given a dataset $\dataset$ of size $N$, we split the dataset using a random partition $\{I_k\}^K_{k=1}$ of dataset indices $[N]$ such that the size of each fold $I_k$ is $N/K$. 

In the first stage of DML-IV, for each fold $k\in [K]$, we learn $\hat{s}_k$ and $\hat{g}_k$ using data $\dataset_{I^c_k}$ with indices $I^c_k\coloneqq[N]\setminus I_k$. $\hat{s}_k\approx\expectE[R\lvert C,Z]$ can be learnt through standard supervised learning using a neural network with inputs $(C,Z)$ and label $R$. For $\hat{g}_k$, we follow~\cite{Hartford2017DeepPrediction} to estimate $F_0(A\lvert C,Z)$, the conditional distribution of $A$ given $(C,Z)$, with $\hat{F}$, and then estimate $\hat{g}$ via
\begin{equation*}
\hat{g}(h,c,z)=\sum_{\dot{A}\sim \hat{F}(A\lvert C,Z)} h(C,\dot{A})\approx \int h(C,A)\hat{F}(A\lvert C,Z) dA\approx\expectE[h(C,A)\lvert c,z].
\end{equation*}
If the action space is discrete, $\hat{F}$ is a categorical model, e.g., a DNN with softmax output. For a continuous action space, a mixture of Gaussian models is adopted to estimate the distribution $F_0(A\lvert C,Z)$, where a DNN is used to predict the means and standard deviations of the Gaussian distributions.

In the second stage of DML-IV, we estimate $\hat{\theta}$ using our Neyman orthogonal score function $\orthoM$ in~\cref{eq:neyman_score}. The key here is to optimise $\hat{\theta}$ with data from the $k$-th fold using nuisance parameters $\hat{s}_k$, $\hat{g}_k$ that are trained with data $\dataset_{I^c_k}$, the complement of the data from the $k$-th fold. This is important to fully debias the estimator $\hat{\theta}$. We alternate between the $K$ folds while sampling a mini-batch $(c_i^k,z_i^k)$ of size $n_b$ from each fold $k$ of the dataset to update $\hat{\theta}$ by minimising the empirical loss on the mini-batch following our Neyman orthogonal score $\orthoM$,
\begin{equation*}
\hat{\expectE}_{(c_i^k,z_i^k)} \left[(\hat{s}_k(c,z)-\hat{g}_k(h_\theta,c,z))^2\right] = \sum_{ (c_i^k,z_i^k)}\frac{1}{n_b}\left((\hat{s}_k(c,z)-\hat{g}_k(h_\theta,c,z))^2\right ).
\end{equation*}
When the second stage converges, we return the DML-IV estimator $h_{\hat{\theta}}$.

To obtain the DML convergence rate guarantees~\cite{Chernozhukov2018Double/debiasedParameters} for $h_{\hat{\theta}}$, i.e., for $\hat{\theta}$ to converge to the true parameters $\theta_0$ at the rate of $O(N^{-1/2})$ with high probability, there are two key conditions: i) Neyman orthogonality of the score function, and ii) the nuisance parameters should converge to their true values at the crude rate of $o(N^{-1/4})$. The Neyman orthogonal score is given in~\cref{thm:neyman}, so it remains to prove the convergence rate of the nuisance parameters. Define $\mathcal{G}_N$ to be the \textit{realisation set} such that $\hat{g}_N$, the estimator of $g_0$ using a dataset of size $N$, takes values in this set. Similarly, define $\mathcal{S}_N$ to be the \textit{realisation set} of $\hat{s}_N$. These realisation sets are properly shrinking neighbourhoods of the true functions $g_0$ and $s_0$, and we later provide~\cref{lemma:informal_nuisances} that describes the rate of shrinkage of these realisation sets, for which we require boundedness of functions $g, s, h$ and the outcome variable $R$ as stated in~\cref{assump:dml}.
\begin{assumption}\label{assump:dml}
We assume that (a): $g_0, s_0, h_0\in\mathcal{G},\mathcal{S},\mathcal{H}$ are all bounded i.e.,\\$\norm{g_0}_\infty, \norm{s_0}_\infty, \norm{h_0}_\infty \leq B$; and (b): the outcome $\norm{R}_\infty\leq B$, where $B\in\realNumber^+$.
\end{assumption}
To improve readability, we provide here an informal statement of the lemma, which expresses the relationship between the critical radius~\cite{Wainwright2019,bartlett2005local} of the realisation sets and the convergence rate of the nuisance parameters. We defer the formal statement and the proof to~\cref{appen:dml}.

\begin{lemma}[Informal: nuisance parameters convergence.]\footnote{See~\cref{lemma:nuisances} for the formal statement.}\label{lemma:informal_nuisances}
If~\cref{assump:dml} holds, let $\delta_N$ be an upper bound on the critical radius of the function spaces related to the realisation sets $\mathcal{S}_N$ and $\mathcal{G}_N$. Then, with probability $1-\zeta$:
\begin{align*}
    \norm{\hat{s}-s_0}_2^2&=O\left(\delta_N^2+\frac{\ln(1/\zeta)}{N}\right);\\
    \norm{\hat{g}-g_0}_2^2&=O\left(\delta_N^2+\frac{\ln(1/\zeta)}{N}\right).
\end{align*}
\end{lemma}
The critical radius is a quantity that describes the complexity of estimation, and it is typically shown that $\delta_N=O(d_N N^{-1/2})$~\cite{Chernozhukov2022RieszNetForests,Chernozhukov2021AutomaticRegression}, where $d_N$ is the effective dimension of the hypothesis space (see~\cref{appen:critical_radius} for the derivation and formal definitions). This, together with~\cref{lemma:informal_nuisances}, implies that $\norm{\hat{s}-s_0}_2=O(d_N N^{-1/2})$. Therefore, for function classes with $d_N=o(N^{1/4})$, $\norm{\hat{s}-s_0}_2\leq o(N^{-1/4})$ (and similarly for $\hat{g}$). This is a broad class of functions that covers many machine learning methods such as deep ReLU networks and shallow regression trees~\cite{Chernozhukov2021AutomaticRegression}. It has also been shown that conditional density and expectation estimation used for $\hat{g}$ satisfies $d_N=o(N^{1/4})$ under mild assumptions~\cite{Grunewalder2018,Bilodeau2021}. We refer to \citet{Chernozhukov2021AutomaticRegression} for additional discussion and concrete convergence rates of nuisance estimators.


\cref{lemma:informal_nuisances} shows that the nuisance parameters converge to their true values at the rate of $o(N^{-1/4})$ if $d_N=o(N^{1/4})$, thus satisfying the second key condition to get the DML convergence rate guarantees. This allows us, after checking some mild regularity and continuity conditions, to obtain the following theorem regarding the convergence of the DML-IV estimator by applying Theorem 3.3 of~\citet{Chernozhukov2018Double/debiasedParameters}, with proof deferred to~\cref{appen:dml}.


\begin{theorem}[Convergence of the DML-IV estimator]\label{thm:dml}
If the effective dimension $d_N=o(N^{1/4})$ for $\hat{s}$, $\hat{g}$, and Assumption~\ref{assump:struc}, \& \ref{assump:dml} hold, we have that the DML-IV estimator $\hat{\theta}$ is concentrated in a $1/\sqrt{N}$ neighbourhood of $\theta_0$, and is approximately linear and centred Gaussian:
\begin{align*}
    \sqrt{N}(\hat{\theta}-\theta_0)\rightarrow \mathcal{N}(0,\sigma^2) \text{ in distribution},
\end{align*}
where the estimator variance is given by
\begin{equation*}
\sigma^2 \coloneqq J_0^{-1}\expectE[\orthoM(\dataset,\theta_0,\eta_0)\orthoM(\dataset,\theta_0,\eta_0)^T](J_0^{-1})^T,
\end{equation*}
which is constant w.r.t $N$ and $J_0$ denotes the Jacobian matrix of $\expectE[\orthoM]$ w.r.t $\theta$.
\end{theorem}
\cref{thm:dml} states that, with adequately trained nuisance parameter estimators, the estimator error $\hat{\theta}-\theta_0$ is normally distributed and variance shrinks at the rate of $N^{-1/2}$. This implies that $\hat{\theta}$ converges to $\theta_0$ at the rate $O(N^{-1/2})$ with high probability, which allows us to deduce suboptimaltiy bounds for the policy induced by $h_{\hat{\theta}}$ in the next section.


\subsection{Suboptimality Bounds}
From the DML-IV estimator $h_{\hat{\theta}}$, we retrieve (an estimate of) the induced optimal policy as $\hat{\policy}(c)\coloneqq\argmax_a h_{\hat{\theta}}(c,a)$. Recall that the suboptimality of a policy is $\textrm{subopt}(\hat{\policy})\coloneqq\Value(\policy^*)-\Value(\hat{\policy})$. Next, we show a suboptimality bound for the DML-IV policy in terms of the sample size $N$.

\begin{theorem}[Suboptimality Bounds]\label{thm:subopt}
Let the learnt policy from a dataset of size $N$ be $\hat{\pi}(c)\coloneqq\argmax_a h_{\hat{\theta}}(c,a)$, where $\hat{\theta}$ is the DML-IV estimator. Let $L$ be a constant such that $\abs{h_\theta(C,A)-h_{\theta^\prime}(C,A)}\leq L \norm{\theta-\theta^\prime}$ for all $C$ in the support of $\testD$, $A\in\mathcal{A}$, and $\theta,\theta^\prime\in\Theta$. Then, for all $\zeta\in(0,1]$, we have that the suboptimality of $\hat{\pi}$ satisfies
\begin{align*}
\textrm{subopt}(\hat{\pi})=O\left(L\sqrt{\frac{\ln(1/\zeta)}{N}}\right),
\end{align*}
with probability $1-\zeta$.
\end{theorem}
The proof is deferred to~\cref{appen:subopt}. To the best of our knowledge, this is the first time that the convergence rate and suboptimality bounds of $O(N^{-1/2})$ have been proved for IV regression methods that use DL, matching the suboptimality bounds of the unconfounded bandit. On the other hand, most other DL-based IV regression methods only demonstrate that their estimators converge in the limit.
\section{Experimental Results}\label{sec:exp}

\begin{figure*}[t]
\centering
\begin{subfigure}[t]
{1\textwidth}
\centering
\includegraphics[width=0.6\textwidth]{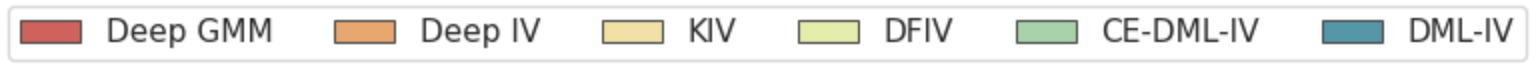}
\end{subfigure}
\begin{subfigure}[t]
{0.32\textwidth}
\centering
\includegraphics[width=1\textwidth]{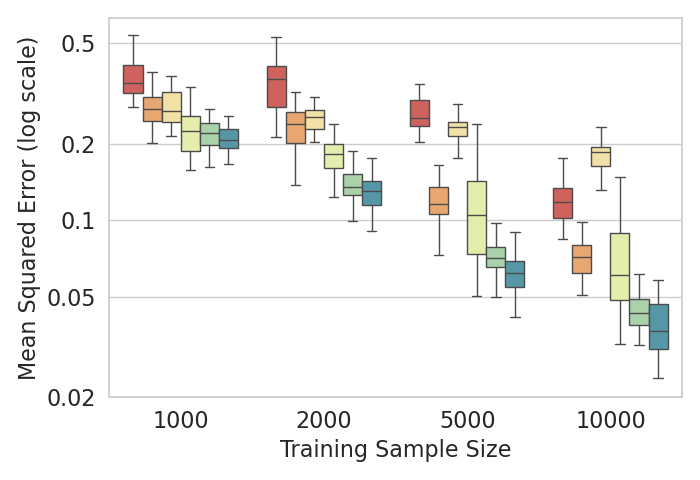}
    \caption{The mean squared error of $\hat{h}$.}
    \label{fig:lowd_mse}
\end{subfigure}
\begin{subfigure}[t]{0.32\textwidth}
\centering\captionsetup{width=0.9\linewidth}\includegraphics[width=1\textwidth]{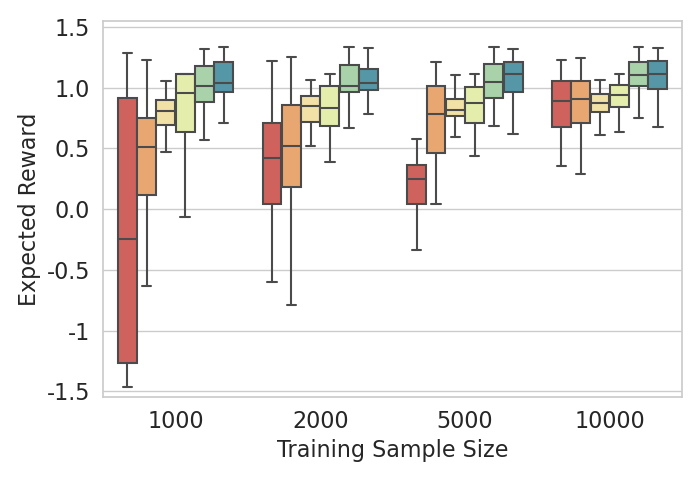}
\caption{The average reward following the policy $\hat{\policy}$ derived from $\hat{h}$.}
    \label{fig:lowd_r}
\end{subfigure}
\begin{subfigure}[t]{0.32\textwidth}
\centering\captionsetup{width=.9\linewidth}
\includegraphics[width=1\textwidth]{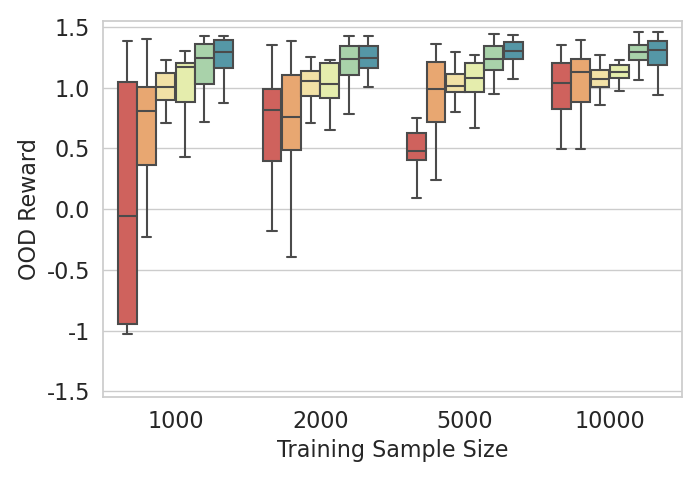}
\caption{The average reward following $\hat{\policy}$ with out of training distribution context.}\label{fig:lowd_ood_r}
\end{subfigure}
\caption{Results on the aeroplane ticket demand dataset with low-dimensional context.}
\end{figure*}

In this section, we empirically evaluate DML-IV for IV regression and offline IV bandit problems. In addition, we evaluate a computationally efficient version of DML-IV, referred to as CE-DML-IV, which does not apply $K$-fold cross-fitting. It trains $\hat{s}$ and $\hat{g}$ only once (instead of $K$ times) using the entire dataset, and can also be considered as an ablation study on $K$-fold cross-fitting. Without $K$-fold cross-fitting, it lacks the theoretical convergence rate guarantees but it still enjoys the partial debiasing effect~\cite{Mackey2018} from the Neyman orthogonal score and trades off computational complexity with bias. We found that CE-DML-IV empirically performs as well as standard DML-IV on low-dimensional datasets. We provide details and discussion regarding CE-DML-IV in~\cref{appen:cedml}.

Our evaluation considers both low- and high-dimensional contexts, as well as semi-synthetic real-world datasets. We compare our methods with leading modern IV regression methods Deep IV~\cite{Hartford2017DeepPrediction}, DeepGMM~\cite{Bennett2019DeepAnalysis}, KIV~\cite{Singh2019} and DFIV~\cite{Xu2020}. In this section we use DNN estimators for both stages with network architecture and hyper-parameters provided in~\cref{appen:networks}. Additional results of DML-IV using tree-based estimators such as Random Forests and Gradient Boosting are provided in~\cref{appen:tree-based}, where SOTA performance is also demonstrated. The algorithms are implemented using PyTorch~\cite{Paszke2019}, and the code is available on GitHub\footnote{\url{https://github.com/shaodaqian/DML-IV}}.

\subsection{Aeroplane Ticket Demand Dataset}

\begin{figure*}[t]
\centering
\begin{subfigure}[t]
{1\textwidth}
\centering
\includegraphics[width=0.6\textwidth]{legend_flat.png}
\end{subfigure}
\begin{subfigure}[t]
{0.32\textwidth}
\centering
\includegraphics[width=1\textwidth]{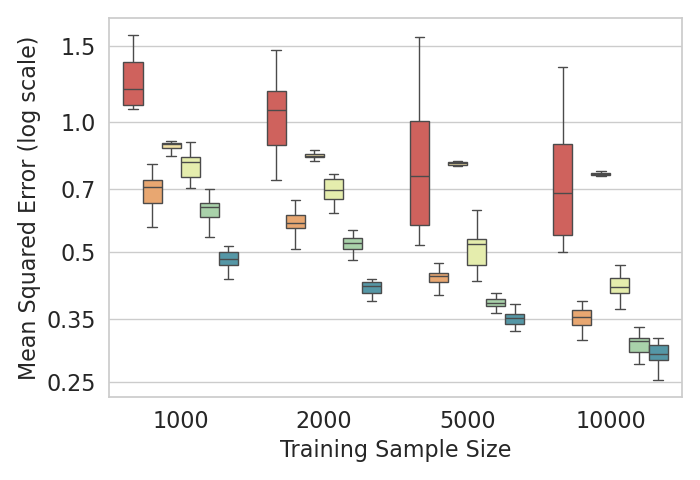}
    \caption{The mean squared error of $\hat{h}$.}
    \label{fig:mnist_mse}
\end{subfigure}
\begin{subfigure}[t]{0.32\textwidth}
\centering\captionsetup{width=0.9\linewidth}\includegraphics[width=1\textwidth]{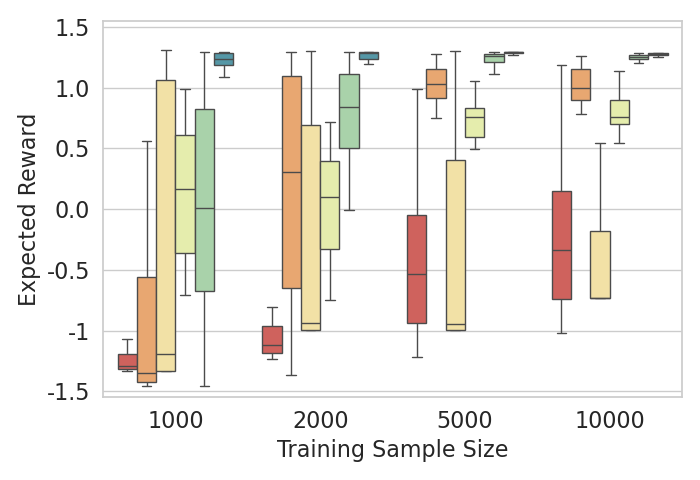}
\caption{The average reward following the policy $\hat{\policy}$ derived from $\hat{h}$.}
    \label{fig:mnist_r}
\end{subfigure}
\begin{subfigure}[t]{0.32\textwidth}
\centering\captionsetup{width=.9\linewidth}
\includegraphics[width=1\textwidth]{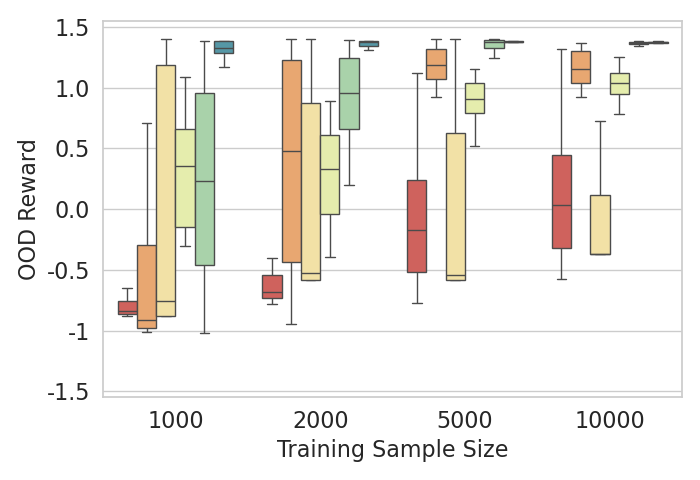}
\caption{The average reward following $\hat{\policy}$ with out of training distribution context.}\label{fig:mnist_ood_r}
\end{subfigure}
\caption{Results on the aeroplane ticket demand dataset with high-dimensional context. }
\end{figure*}

We first conduct experiments for IV regression on the aeroplane ticket demand dataset, which is a synthetic dataset introduced by~\citet{Hartford2017DeepPrediction} that is now a standard benchmark for nonlinear IV methods. In this dataset, we aim to understand how ticket prices $p$ affect ticket sales $r$. We observe two context variables, which are the time of year $t\in[0,10]$ and customer type $s\in[7]$ variables, the latter categorised by the level of price sensitivity. Price and context affect sales through $h_0((t,s),p)=100+(10+p)\cdot s \cdot \psi(t)-2p$, where $\psi(t)$ is a complex nonlinear function. However, the noise of $r$ and $p$ is correlated, which indicates the existence of unobserved confounders. The fuel price $z$ is introduced as an instrumental variable. Details of this dataset are included in~\cref{appen:demand}.

The results for learning $h_0$ with this dataset of various sizes are provided in~\cref{fig:lowd_mse}. We ran each method 20 times and report the mean squared errors (MSE) between the estimators $\hat{h}$ and $h_0$, where the median, 25th and 75th percentiles are shown. It can be seen that DML-IV performs better than other IV regression methods for all dataset sizes. CE-DML-IV, which requires significantly less computation, matches the performance of DML-IV in this case.

\noindent\textbf{High-Dimensional Feature Space} 

In real applications, we
typically do not observe variables such as the customer type as explicit categories. Therefore, we follow~\citet{Hartford2017DeepPrediction} and consider the case where the customer type $s\in[7]$ is replaced by images of the corresponding handwritten digits from the MNIST dataset~\cite{LeCun2010} to evaluate our methods with high-dimensional ($28^2$=784 dimensions) inputs. The task remains to learn $h_0$, but the algorithms are no longer explicitly given the 7 customer types, and instead have to infer the relationship between the image data and the outcome. Results for IV regression are plotted in~\cref{fig:mnist_mse}, where DML-IV and CE-DML-IV outperforms all other methods. In these high-dimensional settings, regularisation is heavily used to avoid overfitting. DML-IV demonstrates the benefits of using DML to reduce both the regularisation and overfitting bias caused by learning the nuisance parameters.

To demonstrate the robustness of DML-IV, we first provide a sensitivity analysis against hyperparameter changes in~\cref{appen:sensitivity}. We evaluate DML-IV and CE-DML-IV on the aeroplane ticket demand datasets under a range of hyperparameters, where  stable performance is observed. In addition, we consider the case when the IV is weakly correlated with the action in~\cref{appen:weak_iv}, where we empirically demonstrate that DML-IV and CE-DML-IV perform significantly better than SOTA methods under weak instruments.

\subsection{Offline IV Bandit}

We also evaluate DML-IV's ability to learn good decision policies in the offline IV bandit problem. We reuse the aeroplane ticket demand dataset and aim to find the best pricing policy that maximises sales. From the learnt $\hat{h}$, for each context sampled from the test distribution, we retrieve the best action by uniformly sampling actions from the action space $\actions$ and selecting the action for which $\hat{h}$ returns the highest value. Using this induced policy $\hat{\policy}$, we compare the expected reward following $\hat{\policy}$ over the test distribution.

For the low-dimensional ticket demand dataset, we first set the test distribution to be the same as the training distribution and plot the average rewards in~\cref{fig:lowd_r}. In~\cref{fig:lowd_ood_r}, we shift the test distribution out of the training distribution by incrementing the distribution of $t$ by $1$. For the high-dimensional setting, \cref{fig:mnist_r} and~\cref{fig:mnist_ood_r} demonstrate the expected rewards for test distributions in and out of the training distribution, respectively. There is a clear trend that a better fitted (low MSE) $\hat{h}$ leads to an induced policy with higher expected reward. In all cases, DML-IV outperforms all other methods, especially in the high-dimensional setting, where DML-IV consistently learns the near-optimal policy with only 2000 samples. CE-DML-IV, on the other hand, only matches the performance of DML-IV for the low-dimensional setting, but still outperforms the other methods in the high-dimensional setting.

We only compare with other IV regression methods because there are no offline bandit methods that consider the IV setting, and standard offline bandit algorithms (e.g.,~\cite{Valko2013,Jin2021,Nguyen-Tang2022}) fail to learn meaningful policies when the dataset is confounded, as demonstrated in~\cref{appen:offline_bandit}.

\subsection{Real-World Decision Problem}

Lastly, we test the performance of DML-IV on real-world datasets. The true counterfactual prediction function is rarely available for real-world data. 
Therefore, in line with previous approaches~\cite{Shalit2017,Wu2023,Schwab2019,Bica2020}, we instead consider two semi-synthetic real-world datasets IHDP\footnote{IHDP: \url{https://www.fredjo.com/}.}~\cite{Hill2011} and PM-CMR\footnote{PM-CMR:\url{https://doi.org/10.23719/1506014}.}~\cite{Wyatt2020}. We directly use the continuous variables from IHDP and PM-CMR as context variables, and generate the outcome variable with a nonlinear synthetic function following~\citet{Wu2023}. There are 470 and 1350 training samples in IHDP and PM-CMR, respectively (for details see~\cref{appen:real}). We also run each method 20 times, where the MSE of $\hat{h}$ and the expected reward of the induced policy $\hat{\policy}$ on the test dataset are plotted in~\cref{fig:real_dataset}. DML-IV and CE-DML-IV demonstrate comparable, if not lower, MSE of fitting $\hat{h}$ than the other methods, while outperforming all other methods in average reward. This shows that our algorithm can reliably learn the counterfactual prediction function and policies with the highest average reward from real-world data.

\begin{figure}[t]
\centering
\begin{subfigure}[t]{1\textwidth}
\centering
\includegraphics[width=0.6\textwidth]{legend_flat.png}
\end{subfigure}
\begin{subfigure}[t]{0.5\textwidth}
\centering
\includegraphics[width=0.6\textwidth]{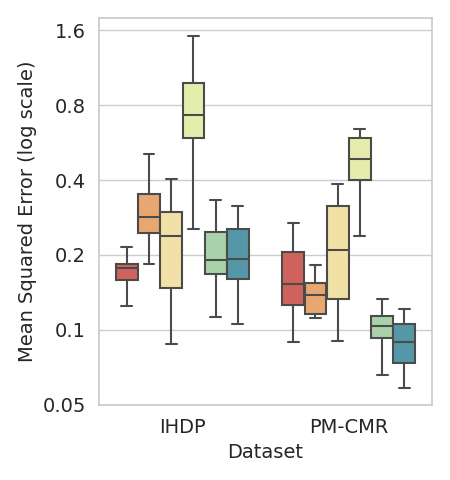}
\end{subfigure}
\begin{subfigure}[t]{0.49\textwidth}
\centering
\includegraphics[width=0.6\textwidth]{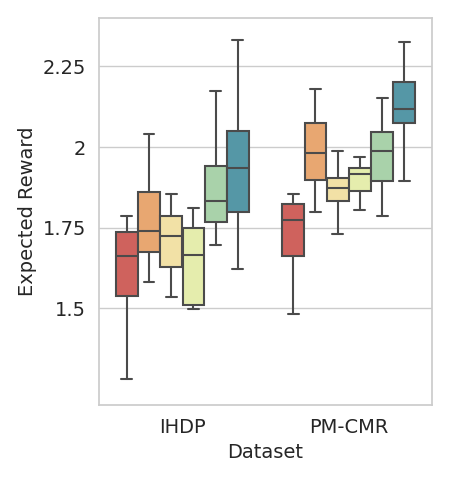}
\end{subfigure}
\vspace{-0.8cm}
\caption{The mean squared error of $\hat{h}$ and average reward following $\hat{\policy}$ for the real-world datasets.}
\label{fig:real_dataset}
\end{figure}
\section{Conclusion}\label{sec:conclusion}

We have proposed a novel method for instrumental variable regression, DML-IV. By leveraging IVs and DML on offline data, DML-IV can learn counterfactual predictions and effective decision policies with fast convergence rate and suboptimality guarantees by mitigating the regularisation and overfitting biases of DL.
We evaluated DML-IV on IV regression benchmarks and IV bandit problems, including semi-synthetic real-world data, experimentally showing it is superior compared to SOTA IV regression methods. 

Future work includes considering other estimation methods for the nuisance parameters following our Neyman-orthogonal score, and extending the method to sequential decision problems and reinforcement learning in the presence of hidden confounders~\cite{Namkoong2020}.

\section*{Acknowledgments}
This work was supported by the EPSRC Prosperity Partnership FAIR (grant number EP/V056883/1). DS acknowledges funding from the Turing Institute and Accenture collaboration. AS was partially supported by AI Singapore, grant
AISG2-RP-2020-018. FQ acknowledges funding from ELSA: European Lighthouse on Secure and Safe AI project (grant agreement
No. 101070617 under UK guarantee). MK receives funding from the ERC under the European Union’s Horizon 2020 research and innovation programme (\href{http://www.fun2model.org}{FUN2MODEL}, grant agreement No.~834115).

\section*{Impact Statement}
The goal of the paper is to develop a methodology to learn high-performing decision policies from offline data. There are many applications of our work in automated decision making, for example, in planning, healthcare, and finance. The theoretical guarantees that we provide ensure the reliability and suboptimality guarantees of the learnt policies. We do not foresee negative implications of our methodology, but would caution against deploying it without human input and recommend additional validation in any new setting to reduce the risk of misapplication.

\bibliographystyle{abbrvnat}
\bibliography{main.bib}

\newpage

\renewcommand{\thesection}{\Alph{section}}
\setcounter{section}{0}
\noindent {\LARGE\textbf{Appendix}}
\section{Computationally Efficient CE-DML-IV}\label{appen:cedml}
\begin{algorithm}[ht]
   \caption{Computationally Efficient CE-DML-IV}
   \label{alg:dml-iv}
\begin{algorithmic}
   \STATE {\bfseries Input:} Dataset $\dataset$ with size $N$, mini-batch size $n_b$
    \STATE {\bfseries Output:} The CE-DML-IV estimator $h_{\hat{\theta}}$
   \STATE Learn $\hat{s}$ and $\hat{g}$ using $\dataset$
   \STATE Initialise $h_{\hat{\theta}}$
   \REPEAT
   \STATE Sample $n_b$ data $(c_i,z_i)$ from $\dataset$
   \STATE $\mathcal{L}=\hat{\expectE}_{(c_i,z_i)}\left[(\hat{s}(c,z)-\hat{g}(h_\theta,c,z))^2\right]$
   \STATE Update $\hat{\theta}$ to minimise loss $\mathcal{L}$ 
    \UNTIL{convergence}
\end{algorithmic}
\end{algorithm}

The standard DML-IV with $K$-fold cross-fitting trains $\hat{s}$ and $\hat{g}$ $K$ times on different subsets of the dataset to tackle overfitting bias, but it is computationally expensive. Therefore, as mentioned in~\cref{sec:exp}, we also evaluate CE-DML-IV, a computationally efficient version of DML-IV that does not apply $K$-fold cross-fitting and trains $\hat{s}$ and $\hat{g}$ only once using the entire dataset. It uses the same Neyman orthogonal score as the standard DML-IV, so it still enjoys the partial debiasing effect~\cite{Mackey2018} from the Neyman orthogonal score. However, without $K$-fold cross-fitting, it lacks the theoretical convergence rate guarantees provided by~\cref{thm:dml} and \cref{thm:subopt}. CE-DML-IV can be viewed as a trade-off between computational complexity and theoretical guarantees, and we found that CE-DML-IV empirically performs as well as standard DML-IV on low-dimensional datasets, where overfitting bias is not prevalent.

\section{Standard Loss Function for IV Regression}\label{appen:score}

The standard score (or loss) function for two-stage IV regression is $\ell=(R-g(h,c,z))^2$, as described in~\cref{eq:optim}. This score is not Neyman orthogonal because, first of all, $\expectE[(R-g_0(h_0,c,z))^2]=\expectE[(R-\expectE[R\lvert C,Z])^2]\neq0$ since $\expectE[h_0\lvert C,Z]=\expectE[R\lvert C,Z]$ and $R-\expectE[R\lvert C,Z]\neq0$ due to the noise on $R$.

Secondly, the derivative against small changes in $g$ for score $\expectE[(R-g_0(h_0,c,z))^2]$ is
\begin{align*}
\frac{\partial}{\partial r}\expectE\Bigl[&(R-g_0(h_0,C,Z)-r\cdot g(h_0,C,Z))^2\Bigr]\\
=&\frac{\partial}{\partial r}\expectE\Bigl[(R-g_0(h_0,C,Z))^2-2r\cdot(R-g_0(h_0,C,Z))g(h_0,C,Z)+r^2\cdot g(h_0,C,Z)^2\Bigr]\\
=&\expectE\Bigl[2(R-g_0(h_0,C,Z))g(h_0,C,Z)+2r\cdot g(h_0,C,Z)^2
\Bigr],
\end{align*}
and, when $r=0$, this derivative evaluates to
\begin{equation*}
\expectE[2(R-g_0(h_0,c,z))g(h_0,c,z)]=\expectE[2
(R-\expectE[R\lvert C,Z])g(h_0,c,z)]
\end{equation*}
which does not equal to 0 for general $g\in\mathcal{G}$ since generally $g(h_0,c,z)$ and the residual $(R-\expectE[R\lvert C,Z])$ are correlated. Therefore, the standard score function for two-stage IV regression can not be used to create a DML estimator.

\section{Omitted Proofs}
In this section, we state all the conditions required to prove the $N^{-1/2}$ convergence rate guarantees for the DML-IV estimator, and provide the omitted proofs in the main paper for~\cref{thm:neyman},~\cref{lemma:informal_nuisances},~\cref{thm:dml} and~\cref{thm:subopt}.

\subsection{DML-IV $N^{-1/2}$ Convergence Rate Guarantees}\label{appen:dml}

To obtain $N^{-1/2}$ convergence rate guarantees of the DML-IV estimator, the following conditions must be satisfied.

\begin{condition}[Conditions for $N^{-1/2}$ convergence of DML, Assumption 3.3 and 3.4 in~\citet{Chernozhukov2018Double/debiasedParameters}]\label{condition:dml}
For $N\geq3$, all the following conditions hold. (a): The true parameter $\theta_0$ obeys $\expectE[\orthoM(\dataset;h_0,(s_0,g_0))]=0$ and $\Theta$ contains a ball of radius $c_1 N^{-1/2}\log N$ centered at $\theta_0$. (b): The map $(\theta,(s,g))\mapsto \expectE[\orthoM(\dataset;h_{\theta},(s,g))]$ is twice continuously Gateaux-differentiable. (c): For all $\theta\in\Theta$, the identification relationship
\begin{align}
\norm{\expectE[\orthoM(\dataset;h_{\theta},(s_0,g_0))]}\gtrsim \norm{J_0(\theta-\theta_0)}
\end{align}
is satisfied, where $J_0\coloneqq\partial_{\theta^\prime}\{\expectE[\orthoM(\dataset;h_{\theta^\prime},(s_0,g_0))]\}|_{\theta^\prime=\theta_0}$ is the Jacobian matrix, with singular values strictly positive (bounded away from zero). 
(d): The score $\orthoM$ obeys the Neyman orthogonality. (e): Let $K$ be a fixed integer. Given a random partition $\{I_k\}_{k=1}^K$ of indices $[N]$ each of size $n=N/K$, we have that the nuisance parameter estimator $\hat{\eta}$ learnt using data with indices $I^c_k$ belongs to a shrinking realisation set $\mathcal{T}_N$, and the nuisance parameters should be estimated at the $o(N^{-1/4})$ rate, i.e., $\norm{\hat{\eta}-\eta_{0}}_2=o(N^{-1/4})$. (f): All eigenvalues of the matrix $\expectE[\orthoM(\dataset;h_{\theta_0},(s_0,g_0))\orthoM(\dataset;h_{\theta_0},(s_0,g_0))^T]$ are strictly positive (bounded away from zero).
\end{condition}

We will check all these conditions in~\cref{thm:neyman},~\cref{lemma:nuisances} and~\cref{thm:dml}.

\begin{proof}[Proof of\textbf{~\cref{thm:neyman}}:]\label{appen:neyman}
Firstly, by Equation~\ref{eq:h_exp}, we have $s_0(C,Z)=g_0(h_0,C,Z)$, thus
\begin{equation*}
\orthoM(\dataset;h_0,(s_0,g_0))=\expectE\Bigl[(s_0(C,Z)-g_0(h_0,C,Z))^2\Bigr]=0
\end{equation*}
Then we compute the derivative w.r.t. small changes in the nuisance parameters. For all $s,g\in \mathcal{S}, \mathcal{G}$,
\begin{align*}
\frac{\partial}{\partial r}&\expectE\Bigl[(s_0(C,Z)+r\cdot s(C,Z)-g_0(h_0,C,Z)-r\cdot g(h_0,C,Z))^2\Bigr]\\
=&\frac{\partial}{\partial r}\expectE\Bigl[2r(s_0(C,Z)-g_0(h_0,C,Z))(s(C,Z)-g(h_0,C,Z))+r^2(s(C,Z)-g(h_0,C,Z))^2\Bigr]\\
=&\expectE\Bigl[2(s_0(C,Z)-g_0(h_0,C,Z))(s(C,Z)-g(h_0,C,Z))+2r(s(C,Z)-g(h_0,C,Z))^2\Bigr],
\end{align*}
and, when at $r=0$, the derivative evaluates to
\begin{align*}
    \expectE\Bigl[&2(s_0(C,Z)-g_0(h_0,C,Z))(s(C,Z)-g(h_0,C,Z))\Bigr]\\
    &=\expectE\Bigl[0 \times(s(C,Z)-g(h_0,C,Z))\Bigr]\\
    &=0 \quad \forall s,g\in \mathcal{S}, \mathcal{G},
\end{align*}
since $s_0(C,Z)=\expectE[R|C,Z]=\expectE[h_0|C,Z]=g_0(h_0,C,Z)$. Therefore, our moment function $\orthoM$ is Neyman orthogonal at $(h_0,(s_0,g_0))$.
\end{proof}

\begin{lemma}[Formal version of~\cref{lemma:informal_nuisances}: Nuisances parameters convergence]\label{lemma:nuisances}
If Assumption~\ref{assump:dml} holds, let $\delta_N$ be an upper bound on the critical radius of the two following function spaces:
\begin{align}
\{(C,Z)\mapsto\gamma(s(C,Z)-s_0(C,Z)): s\in\mathcal{S}_N,\gamma\in[0,1]\};\label{eq:s_space}\\
\{(C,Z)\mapsto\gamma(g(C,Z,h_0)-g_0(C,Z,h_0)): g\in\mathcal{G}_N,\gamma\in[0,1]\},\label{eq:g_space}
\end{align}
and suppose that all functions $f$ in the two spaces above satisfy $\norm{f}_\infty\leq B$ for some $B\in\realNumber^+$.
Then, for some universal constants $c_1$ and $c_2$, we have that with probability $1-\zeta$:
\begin{align*}
    \norm{\hat{s}-s_0}_2^2&\leq c_1\left(\delta_N^2+\frac{B^2\log(1/\zeta)}{N}+\inf_{s_*\in\mathcal{S}_N}\norm{s_*-s_0}_2^2\right);\\
    \norm{\hat{g}-g_0}_2^2&\leq c_2\left(\delta_N^2+\frac{B^2\log(1/\zeta)}{N}+\inf_{g_*\in\mathcal{G}_N}\norm{g_*-g_0}_2^2\right).
\end{align*}
\end{lemma}

\begin{proof}[Proof of\textbf{~\cref{lemma:nuisances}}:]

We will mainly use the result from Theorem 1 of~\citet{Chernozhukov2021AutomaticRegression}, which states the following. For a function $\alpha$ that is the minimizer of a loss function that can be represented as $\expectE[-2m(\dataset,\alpha)+\alpha(x)^2]$, where $\dataset$ is the offline dataset and $m$ is some moment function that satisfies 
\begin{equation*}
\expectE[(m(W,\alpha)-m(W,\alpha^\prime))^2]\leq M\norm{\alpha-\alpha^\prime}_2^2\quad\forall\alpha,\alpha^\prime\in\mathcal{A}_N.
\end{equation*}
Let $\delta_N$ be an upper bound on the critical radius of the two function spaces:
\begin{align*}
\{W\mapsto\gamma(\alpha(W)-\alpha_0(W)): \alpha\in\mathcal{A}_N,\gamma\in[0,1]\};\\
\{W\mapsto\gamma(m(W,\alpha)-m(W,\alpha_0)): \alpha\in\mathcal{A}_N,\gamma\in[0,1]\}.
\end{align*}
Then, if $\norm{\alpha}_\infty\leq B$ for some $B\in\realNumber^+$, there exists a universal constant $c$ such that with probability $1-\zeta$,
\begin{align*}
\norm{\hat{\alpha}-\alpha_0}_2^2\leq c \left(\delta_N^2+\frac{M\log(1/\zeta)}{N}+\inf_{\alpha_*\in\mathcal{A}_N}\norm{\alpha_*-\alpha_0}^2_2\right).
\end{align*}
In our case, we show that the loss function for both $s$ and $g$ satisfies the above conditions, and thus Theorem 1 of \citet{Chernozhukov2021AutomaticRegression} is applicable to provide an upper bound on the convergence rate of our nuisance parameters.\\

The loss function for $s\in\mathcal{S}_N$ is 
\begin{align*}
s_0&=\argmin_{s\in\mathcal{S}}\expectE\left[(R-s(C,Z))^2\right]\\
&=\argmin_{s\in\mathcal{S}}\expectE\left[R^2-2Rs(C,Z)+s(C,Z)^2\right]\\
&=\argmin_{s\in\mathcal{S}}\expectE\left[-2Rs(C,Z)+s(C,Z)^2\right],
\end{align*}
where we can set $m(W,s)=Rs(C,Z)$ and check that
\begin{align*}
\expectE[(Rs(C,Z)-Rs_0(C,Z))^2] &\leq \expectE[R^2(s(C,Z)-s_0(C,Z))^2]\\
&\leq\norm{R^2}_{\infty}\expectE[(s(C,Z)-s_0(C,Z))^2]\\
&=B^2 \norm{s(C,Z)-s_0(C,Z)}_2^2,
\end{align*}
by Hölder's inequality and the assumption that $\norm{R}_\infty\leq B$. Therefore, by Theorem 1 of~\citet{Chernozhukov2021AutomaticRegression}, there exists a universal constant $c_1$ such that with probability $1-\zeta$,
\begin{align*}
    \norm{\hat{s}-s_0}_2^2\leq c_1\left (\delta_N^2+\frac{B^2 \log(1/\zeta)}{N}+\inf_{s_*\in\mathcal{S}_N}\norm{s_*-s_0}^2_2\right ),
\end{align*}
where recall $\delta_N$ is an upper bound on the critical radius of the function spaces defined in Eq. \eqref{eq:s_space}.
\\

For the second part of the proof, recall that
\begin{align*}
g(h,c,z)&=\int h(C,A)F(A\mid C,Z) dA,
\end{align*}
where $F(A\mid C,Z)$ is some distribution over $A$ and $F_0(A\mid C,Z)=\probP(A\mid C,Z)$ is the distribution of $A$ conditional on $(C,Z)$. Therefore, $g_0$ should minimise the following loss:
\begin{align*}
g_0&=\argmin_{g\in\mathcal{G}}\expectE\left[\left(\int h_0(C,A)\probP(A\mid C,Z) dA-g(C,Z,h_0)\right)^2\right]\\
&=\argmin_{g\in\mathcal{G}}\expectE\left[-2\int h_0(C,A)\probP(A\mid C,Z) dA\cdot g(C,Z,h_0)+g(C,Z,h_0)^2\right],
\end{align*}
where we can set $m(\dataset,g)=\int h_0(C,A)\probP(A\mid C,Z) dA\cdot g(C,Z,h_0)$ and check that
\begin{align*}
\expectE&\left[\left(\int h_0(C,A)\probP(A\mid C,Z) dA\cdot g(C,Z,h_0)-\int h_0(C,A)\probP(A\mid C,Z) dA\cdot g_0(C,Z,h_0)\right)^2\right]\\
=&\expectE\left[\left(\int h_0(C,A)\probP(A\mid C,Z) dA\right)^2\cdot \left(g(C,Z,h_0)-g_0(C,Z,h_0)\right)^2\right]\\
=&\expectE\left[g_0(C,Z,h_0)^2\cdot \left(g(C,Z,h_0)-g_0(C,Z,h_0)\right)^2\right]\\
\leq & \norm{g_0^2}_{\infty} \norm{g(C,Z,h_0)-g_0(C,Z,h_0)}_2^2 \\
\leq & B^2 \norm{g(C,Z,h_0)-g_0(C,Z,h_0)}_2^2,
\end{align*}
by Hölder's inequality, where $M$ is a constant since $g$ is bounded. Therefore, by Theorem 1 of~\citet{Chernozhukov2021AutomaticRegression}, there exists a universal constant $c_2$ such that with probability $1-\zeta$,
\begin{align*}
\norm{g-g_0}_2^2\leq c_2\left (\delta_N^2+\frac{B^2 \log(1/\zeta)}{N}+\inf_{g_*\in\mathcal{G}_N}\norm{g_*-g_0}^2_2\right ),
\end{align*}
where again $\delta_N$ is an upper bound on the critical radius of the function spaces defined in Equation~\ref{eq:g_space}, which completes the proof.
\end{proof}

Now, we are ready to prove~\cref{thm:dml}, which is our main theorem that states the $N^{-1/2}$ convergence rate guarantees for the DML-IV estimator.

\begin{proof}[Proof of\textbf{~\cref{thm:dml}}:]\label{appen:dml_guarantee}
We mainly use Theorem 3.3 from~\citet{Chernozhukov2018Double/debiasedParameters}, where properties of the DML estimator for non-linear scores are demonstrated. It states that, if Condition~\ref{condition:dml} holds, the DML estimator $\hat{\theta}$ is concentrated in a $1/\sqrt{N}$ neighbourhood of $\theta_0$:
\begin{align*}
    \frac{\sqrt{N}}{\sigma}(\hat{\theta}-\theta_0)=\frac{1}{\sqrt{N}}\sum{\bar{\orthoM}(\dataset_i)+O(\rho_N)}\rightarrow \mathcal{N}(0,1) \text{ in distribution},
\end{align*}
where $\bar{\orthoM}(\cdot)\coloneqq -\sigma^{-1}J_0^{-1}\orthoM(\cdot,\theta_0,\eta_0)$ is the influence function, $J_0$ is the Jacobian of $\orthoM$, the approximate variance is $\sigma^2 \coloneqq J_0^{-1}\expectE[\orthoM(\dataset,\theta_0,\eta_0)\orthoM(\dataset,\theta_0,\eta_0)^T](J_0^{-1})^T$, and the size of the remainder $\rho_N$ converges to 0. Therefore, we only need to check whether, under Assumption~\ref{assump:struc} and~\ref{assump:dml}, all of Condition~\ref{condition:dml} for DML $N^{-1/2}$ convergence rate is satisfied. Conditions (a) and (d) are satisfied by Theorem~\ref{thm:neyman}. Condition (b) is satisfied since $(s-g)^2$ is twice continuously differentiable with respect to $s$ and $g$. 



Condition (c) is a sufficient identifiability condition, which states the closeness of the loss function at point $\theta$ to zero and implies the closeness of $\theta$
to $\theta_0$. This assumption is standard in condition moment problems. To check condition (c), we first point out that under analytical assumptions for $s, g$, and $h$, we can write down first order Taylor series for the score function $\expectE[\orthoM(\dataset;h_{\theta},(s_0,g_0))]$ around the point $\theta_0$,
\begin{align*}
\expectE[\orthoM(\dataset;h_{\theta},(s_0,g_0))] = \expectE[\orthoM(\dataset;h_{\theta_0},(s_0,g_0))] + J_0 (\theta - \theta_0) + O(\norm{\theta - \theta_0}^2).
\end{align*}
Plugging in validity of the score function $\orthoM(\dataset;h_{\theta},(s_0,g_0))$, i.e.,  $\expectE[\orthoM(\dataset;h_{\theta_0},(s_0,g_0))] = 0$, we infer that
\begin{align*}
    \norm{\expectE[\orthoM(\dataset;h_{\theta},(s_0,g_0))]}\gtrsim \norm{J_0(\theta-\theta_0)}.
\end{align*}
Now for identifiability, we only need to assume that $J_0 J_0^T$ is non-singular, which is a common technical assumption.

Condition (e) is satisfied since we have that the effective dimension $d_N=o(N^{1/4})$, and together with Lemma~\ref{lemma:nuisances} and the fact that the upper bound of the critical radius $\delta_N=O(d_N N^{-1/2})$ (see~\cref{appen:critical_radius}), the nuisance parameters converge sufficiently quickly to ensure $\norm{\hat{s}-s_0}_2\leq O(\delta_N+N^{-1/2})=O(d_N N^{-1/2})=o(N^{-1/4})$ and $\norm{\hat{g}-g_0}_2\leq O(\delta_N+N^{-1/2})=O(d_N N^{-1/2})=o(N^{-1/4})$. 

Condition (f) is the non-degeneracy assumption for covariance of the score function $\orthoM(\dataset;h_{\theta},(s_0,g_0))$. By definition,
\begin{align*}
    \expectE[\orthoM(\dataset;h_{\theta},(s_0,g_0)) \orthoM(\dataset;h_{\theta},(s_0,g_0))^T] = \int \orthoM(\dataset;h_{\theta},(s_0,g_0)) \orthoM(\dataset;h_{\theta},(s_0,g_0))^T d\probP(\dataset).
\end{align*}
By trace trick, for each datapoint $\dataset$, the only eigenvalue of $\orthoM(\dataset;h_{\theta},(s_0,g_0)) \orthoM(\dataset;h_{\theta},(s_0,g_0))^T $ is $\norm{\orthoM(\dataset;h_{\theta},(s_0,g_0))}^2 \geq 0 $, with $\orthoM(\dataset;h_{\theta},(s_0,g_0))$ as the corresponding eigenvector. Therefore, $\expectE[\orthoM(\dataset;h_{\theta},(s_0,g_0)) \orthoM(\dataset;h_{\theta},(s_0,g_0))^T]$ is positive-definite if for each member $d$ of the support of $\probP$, which is the distribution of $\dataset$, there are at least as many eigenvectors of $d$ as the number of dimension of $\orthoM(\dataset;h_{\theta},(s_0,g_0))$, which is true in our setting as the co-domain of $\orthoM(\dataset;h_{\theta},(s_0,g_0))$ is $\mathbb{R}$.

Therefore, all conditions for Theorem 3.3~\cite{Chernozhukov2018Double/debiasedParameters} to hold are satisfied, which concludes the proof.
\end{proof}

\subsection{Suboptimaltiy}

\begin{proof}[Proof of\textbf{~\cref{thm:subopt}}:]\label{appen:subopt}
From theorem~\ref{thm:dml}, we have that the parameters $\hat{\theta}$ for $h_{\hat{\theta}}$ learned from a dataset of size $N$ using DML-IV satisfy $(\hat{\theta}-\theta_0)\xrightarrow{d}\mathcal{N}(0,\sigma^2/N)$, where $\sigma^2$ is the is the DML-IV estimator variance. This means that, for all $\epsilon>0$ and $\zeta>0$, there exists an integer $K>0$ such that for all $N\geq K$,
\begin{align*}
\probP(\norm{\hat{\theta}-\theta_0}>\epsilon)\leq 1-\Phi\left(\epsilon\cdot \sqrt{N}/\sigma\right)+\zeta/2,
\end{align*}
where $\Phi$ is the CDF of a standard Gaussian distribution.
If we assume $L$ to be a constant such that $\abs{h_\theta(C,A)-h_{\theta^\prime}(C,A)}\leq L \norm{\theta-\theta^\prime}$ for all $C,A\in \textrm{supp}^M(C,A)$ and $\theta\in\Theta$, we have that for all $\epsilon>0$ and $\zeta>0$, there exists an integer $K>0$ such that for all $N\geq K$, 
\begin{equation}
\probP(\abs{h_{\hat{\theta}}(C,A)-h_{\theta_0}(C,A)}>L\cdot\epsilon) \leq 1-\Phi(\epsilon\cdot \sqrt{N}/\sigma)+\zeta/2 \quad \forall C,A\in \textrm{supp}^M(C,A).\label{eq:bound_h}
\end{equation}
Next, we can show that the suboptimality of $\hat{\policy}$ satisfies
\begin{align}
\textrm{subopt}(\hat{\policy})
&=\Value(\policy^*)-\Value(\hat{\policy})\nonumber \\
&=\expectE_{C\sim\testD}[R\mid C,do(A=\policy^*(c))]-\expectE_{C\sim\testD}[R\mid C,do(A=\hat{\policy}(c))]\nonumber \\
&=\expectE_{C\sim\testD}[f_r(C,\policy^*(C))-f_r(C,\hat{\policy}(C))]\nonumber \\
&=\expectE_{C\sim\testD}[h(C,\policy^*(C))-h(C,\hat{\policy}(C))]\nonumber \\
&\leq \max_{c\in \textrm{supp}(\testD)} \left(h(c,\policy^*(c))-h(c,\hat{\policy}(c))\right)\nonumber \\
&\leq \max_{c\in \textrm{supp}(\testD)}\abs{h(c,\policy^*(c))-h_{\hat{\theta}}(c,\policy^*(c))}+(h_{\hat{\theta}}(c,\policy^*(c))-h_{\hat{\theta}}(c,\hat{\policy}(c)))\nonumber\\
&\quad\quad+\abs{h_{\hat{\theta}}(c,\hat{\policy}(c))-h(c,\hat{\policy}(c))}\nonumber \\
&\leq 2L\cdot\epsilon \quad\text{ with probability } \left(\Phi(\epsilon\cdot \sqrt{N}/\sigma)-\zeta/2\right)\label{eq:bound_wp}
\end{align}
where $\textrm{supp}(\testD)$ is the support of $\testD$, by Equation~\ref{eq:bound_h} and the fact that $h_{\hat{\theta}}(C,\policy^*(C))-h_{\hat{\theta}}(C,\hat{\policy}(C))\leq0$. Setting $\Phi(\epsilon\cdot \sqrt{N}/\sigma)=1-\zeta/2$ in Equation~\ref{eq:bound_wp} and substituting $\epsilon$ yields
\begin{equation*}
\textrm{subopt}(\hat{\policy})\leq 2L\Phi^{-1}(1-\zeta/2)\sigma/\sqrt{N}\quad\text{ with probability }1-\zeta.
\end{equation*}
From Blair \textit{et al.}'s approximation for the inverse of the error function (erf)~\cite{Blair1976RationalFunction}, we have that for all $y\in(0,1]$, $\Phi^{-1}(1-y)\leq\sqrt{-2\ln(y)}$. Thus, we conclude that there exists $K>0$ such that for all $N>K$
\begin{equation*}
\textrm{subopt}(\hat{\policy}_N)\leq 2\sqrt{2}L\sigma\sqrt{\frac{\ln(2/\zeta)}{N}}\quad\text{ with probability }1-\zeta,
\end{equation*}
which completes the proof.
\end{proof}

\subsection{Critical Radius and Effective Dimension} \label{appen:critical_radius}

\begin{definition}[\citet{Wainwright2019}]
    The critical radius denoted by $\delta_N$ is defined as the minimum $\delta$ that satisfies the following upper bound on the local Gaussian complexity of a star-shaped function class $\mathcal{F}^*$\footnote{A function class 
  $\mathcal{F}$ is star-shaped if for every $f \in \mathcal{F}$ and $\alpha \in [0, 1]$, we have $\alpha f \in \mathcal{F}$.}, $\mathcal{G(\mathcal{F}^*, \delta)} \leq {\delta^2}/2$, where local Gaussian complexity is defined as
    \begin{align*}
        \mathcal{G(\mathcal{F}^*, \delta)} = \expectE_{\epsilon}[\sup_{g \in \mathcal{F}^*: \norm{g}_N \leq \delta}  \langle \epsilon, g \rangle ],
    \end{align*}
    with $\epsilon$ being a random i.i.d. zero-mean Gaussian vector.
\end{definition}
The critical radius is a standard notion to bound the estimation error in the regression problem. Since local Gaussian complexity can be viewed as an expected value of a supremum of a stochastic process indexed by $g$, we can apply empirical process theory tools, namely the Dudley's entropy integral~\citep{Wainwright2019,van2014probability}, to provide a bound on the critical radius,
\begin{align*}
    \mathcal{G(\mathcal{F}^*, \delta)} \leq \inf_{\alpha \geq 0} \left \{\alpha + \frac{1}{\sqrt{N}}  \int_{\alpha/4}^{\delta} \sqrt{\log \mathcal{N}(\mathcal{F}^*, L^2(P_N), \epsilon)}\:d\epsilon\right \},
\end{align*}
where $\mathcal{N}(\mathcal{F}^*, L^2(P_N), \epsilon)$ is the $\epsilon$-covering number of function class $\mathcal{F}^*$ in $L^2(P_N)$ norm. Now, by placing $\alpha = 0$, when the integral is a single scale value of $ \sqrt{\log \mathcal{N}(\mathcal{F}^*, L^2(P_n), \epsilon)}$, we infer that
\begin{align*}
    \mathcal{G(\mathcal{F}^*, \delta)} \leq \frac{\delta}{\sqrt{N}} \sqrt{\log \mathcal{N}(\mathcal{F}^*, L^2(P_N), \epsilon)}.
\end{align*}
Thus, the critical radius will be upper bounded by
\begin{align*}
    \delta_N \lesssim \frac{\sqrt{\log \mathcal{N}(\mathcal{F}^*, L^2(P_N), \epsilon)}}{\sqrt{N}} = O(d_N N^{-1/2}).
\end{align*}
\citet{Chernozhukov2022RieszNetForests,Chernozhukov2021AutomaticRegression} referred to $d_N = \sqrt{\log \mathcal{N}(\mathcal{F}^*, L^2(P_N), \epsilon)}$ as the effective dimension of the hypothesis space. Note that this matches the minimax lower bound of fixed design estimation for this setting~\citep{yang1999information}.
\section{Datasets Details} 
In this section, we provide details of the datasets considered in this paper.

\subsection{Aeroplane Ticket Demand Dataset}\label{appen:demand}

Here, we describe the aeroplane ticket demand dataset, first introduced by~\citet{Hartford2017DeepPrediction}. The observable variables are generated by the following model:
\begin{align*}
r&=h_0((t,s),p)+\epsilon, \quad \expectE[\epsilon\lvert t,s,p]=0;\\
p&=25+(z+3)\psi(t)+\omega,
\end{align*}
where $r$ is the ticket sales (as the outcome variable) and $p$ is the ticket price (as the action variable). $(t,s)$ are observed context variables, where $t$ is the time of year and $s$ is the customer type. The fuel price $z$ is introduced as an instrumental variable, which only affects the ticket price $p$. The noises $\epsilon$ and $\omega$ are correlated with correlation $\rho\in[0,1]$, where in our experiments we set $\rho=0.9$. $h_0$ is the true counterfactual prediction function, defined as
\begin{align*}
    h_0((t,s),p)&=100+(10+p)\cdot s \cdot \psi(t)-2p,\\
    \psi(t)&=2\left(\frac{(t-5)^4}{600}+\exp(-4(t-5)^2)+\frac{t}{10}-2\right),
\end{align*}
where $\psi(t)$ is a complex non-linear function of $t$ plotted in~\cref{fig:nonlinear}. The offline dataset is sampled with the following distributions:
\begin{align*}
    s&\sim \text{Unif}\{1,...,7\}\\
    t&\sim \text{Unif}(0,10)\\
    z&\sim \mathcal{N}(0,1)\\
    \omega&\sim \mathcal{N}(0,1)\\
    \epsilon&\sim \mathcal{N}(\rho\omega,1-\rho^2).
\end{align*}
From the observations $(r,p,t,s,z)$, we estimate $\hat{h}$ using IV regression methods, and the mean squared error between $\hat{h}$ and the true causal function $h_0$ are computed on 10000 random samples from the above model. For the out of distribution test samples, we sample $t\sim \text{Unif}(1,11)$ instead.

We standardise the action and outcome variables $p$ and $r$ to centre the data around a mean of zero and a standard deviation of one following~\citet{Hartford2017DeepPrediction}. This is standard practice for DNN training, which improves training stability and optimization efficiency.

\begin{figure}[tb]
    \centering
\includegraphics[width=0.5\textwidth]{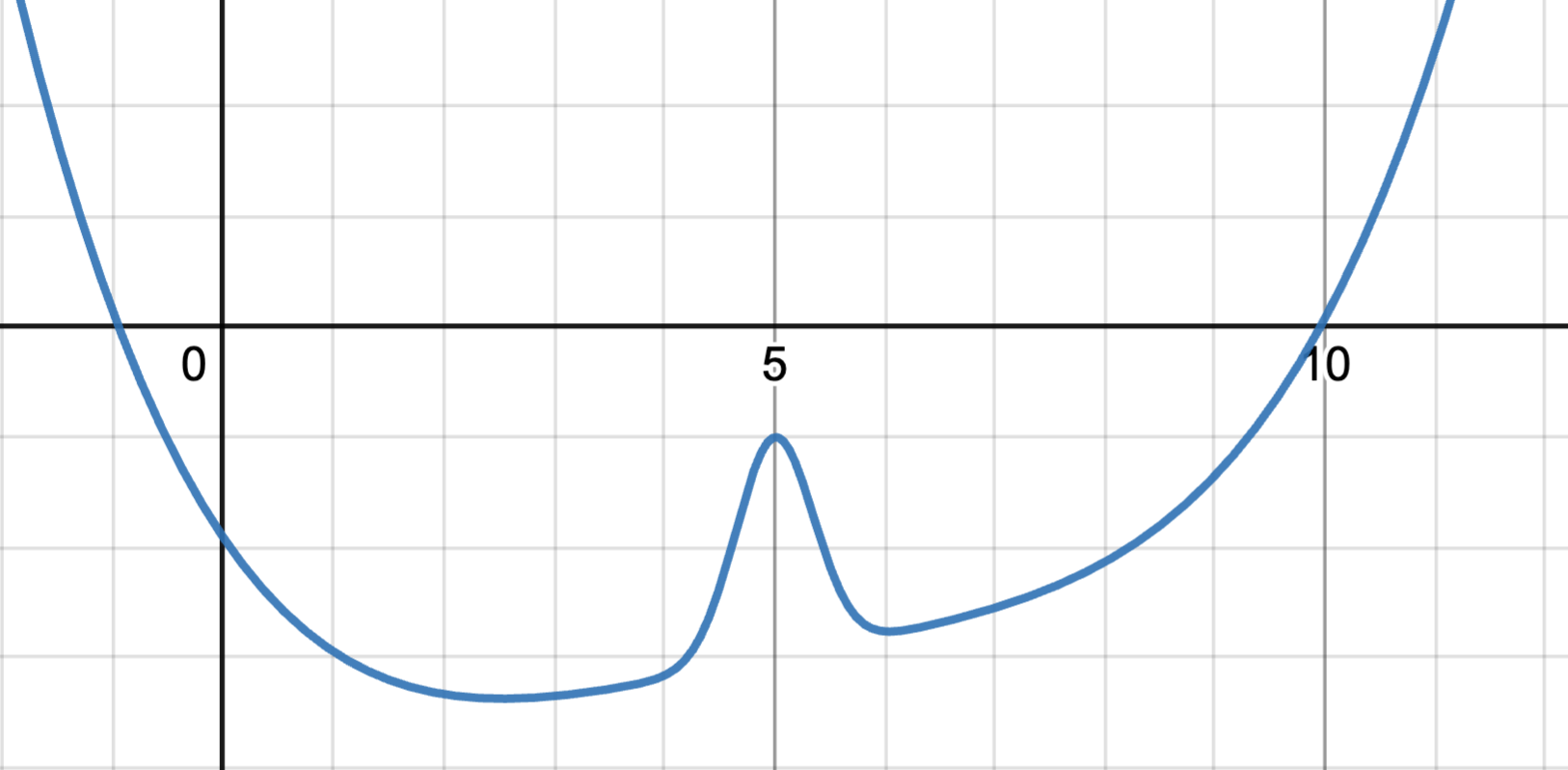}
    \caption{A graph of the nonlinear function $\psi(t)$ in the aeroplane ticket demand dataset.}
    \label{fig:nonlinear}
\end{figure}

\noindent\textbf{High-Dimensional Setting}

For the high-dimensional setting, we again follow~\citet{Hartford2017DeepPrediction} to replace the customer type $s\in[7]$ in the low-dimensional setting with images of the corresponding handwritten digits from the MNIST dataset~\cite{LeCun2010}. For each digit $d\in[7]$, we select a random MNIST image from the digit class $d$ as the new customer type variable $s$. The images are $28\times28=784$ dimensional.

\subsection{Real-World Datasets}\label{appen:real}

Following previously studied causal inference methods~\cite{Shalit2017,Wu2023,Schwab2019,Bica2020}, we consider two semi-synthetic real-world datasets IHDP\footnote{IHDP: \url{https://www.fredjo.com/}.}~\cite{Hill2011} and PM-CMR\footnote{PM-CMR:\url{https://doi.org/10.23719/1506014}.}~\cite{Wyatt2020} for experiments, since the true counterfactual prediction function is rarely available for real-world datasets. 

IHDP, the Infant
Health and Development Program (IHDP), comprises 747
units with 6 pre-treatment continuous variables, one action variable and 19 discrete variables related to the children and their mothers, aiming at evaluating the effect of specialist home visits on the future cognitive test scores of premature infants. From the original data, We select all 6 continuous covariance variables as our context variable $C$.

PM-CMR studies the
impact of PM2.5 particle level on the cardiovascular mortality rate (CMR) in 2132 counties in the United States using data provided by the National Studies on Air Pollution and Health~\cite{Wyatt2020}. We use 6 continuous variables
about CMR in each city as our context variable $C$.

Following~\citet{Wu2023}, from the context variables $C$ obtained from real-world datasets, we generate the instrument $Z$, the action $A$ and the outcome $R$ using the following model:
\begin{align*}
&Z\sim\probP(Z=z)=1/K,\quad z\in[1..K];\\
&A=\sum_{z=1}^K 1_{Z=z} \sum_{i=1}^{d_C}w_{iz}(C_i+0.2\epsilon+f_z(z))+\delta_A, \quad w_{iz}\sim\text{Unif}(-1,1);\\
&R=9A^2-1.5A+\sum_{i=1}^{d_C} \frac{C_i}{d_C}+\abs{C_1 C_2}-\sin{(10+C_2 C_3)}+2\epsilon+\delta_R,
\end{align*}
where $C_i$ denotes the $i$-th variable in $C$, $f_z$ is a function that returns different constants depending on the input $z$, $\delta_R,\delta_A\sim\mathcal{N}(0,1)$ and $\epsilon\sim\mathcal{N}(0,0.1)$ is the unobserved confounder. The fully generated semi-synthetic datasets IHDP and PM-CMR have 747 and 2132 samples respectively, and we randomly split them into training (63\%), validation (27\%), and
testing (10\%) following~\citet{Wu2023}.

\section{Failure of Standard Offline Bandit Algorithms}\label{appen:offline_bandit}

\begin{figure}[t]
\centering
\includegraphics[width=0.7\textwidth]{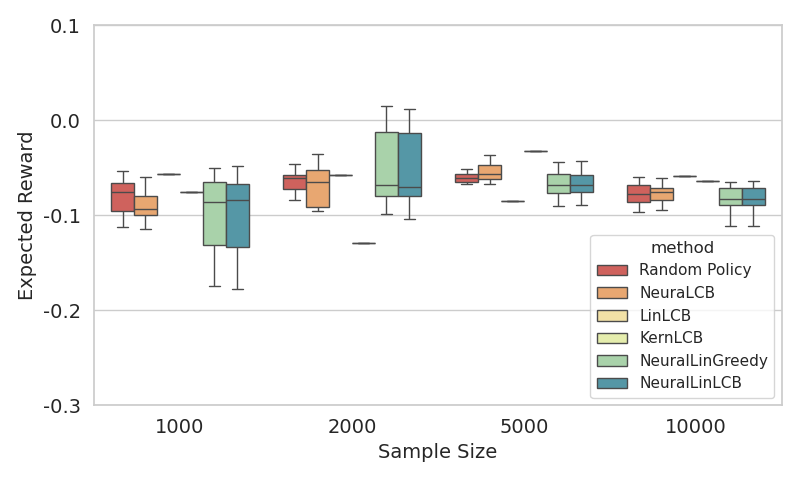}
\caption{Comparing the average reward obtained by policies learned using offline bandit algorithms that do not take IVs into account with a random policy on the aeroplane ticket demand dataset with low-dimensional context.}
\label{fig:offline_bandit}
\end{figure}

It has been demonstrated that standard supervised learning that does not 
take IVs into account fails to learn the causal function or the counterfactual prediction function from a confounded offline dataset~\cite{Hartford2017DeepPrediction}. Similarly, we demonstrate here that standard offline bandit algorithms also fail to learn meaningful policies from confounded offline datasets. We evaluate PEVI, also called LinLCB~\cite{Jin2021}, NeuraLCB~\cite{Nguyen-Tang2022}, KernLCB~\cite{Valko2013}, NeuralLinLCB~\cite{Nguyen-Tang2022} and NeuralLinGreedy~\cite{Nguyen-Tang2022} algorithms, for which we combine the context $C$ and instrument $Z$ variables together as the new context input for these offline bandit algorithms. For algorithms that only support discrete actions, we discretise the action space $\mathcal{A}$ into 20 discrete actions.

For all methods, we follow the network architecture and hyper parameters from the original papers, and we adopt the implementation\footnote{\url{https://github.com/thanhnguyentang/offline_neural_bandits}} of~\citet{Nguyen-Tang2022}. We evaluate these methods on the aeroplane ticket demand dataset described in~\cref{appen:demand} and compare the average reward obtained by the learned policies with a random policy in~\cref{fig:offline_bandit}. It can be seen that all the offline bandit algorithms do not outperform a random policy while DML-IV achieves an average reward higher then 1 as shown in~\cref{fig:lowd_r}. This is unsurprising because these bandit methods do not exploit IVs explicitly and are unable to learn the true causal effect of actions.

\section{Network Structures and Hyper-Parameters}\label{appen:networks}

Here, we describe the network architecture and hyper-parameters of all experiments. Unless otherwise specified, all neural network algorithms are optimised using AdamW~\cite{Loshchilov2017} with learning rate $= 0.001$, $\beta =(0.9,0.999)$ and $\epsilon=10^{-8}$. In addition, we set $K=10$ for $K$-fold cross-fitting in DML-IV.

\subsection{Aeroplane Ticket Demand Dataset}
For DML-IV and CE-DML-IV, we use the network architecture described in~\cref{tab:demand_arch}. We use a learning rate of $0.0002$ with a weight decay of $0.001$ (L2 regularisation) and a dropout rate of $\frac{1000}{5000+N}$ that depends on the data size $N$. For DeepGMM, we use the same structure as the outcome network of DML-IV with dropout $=0.1$ and the same learning rate as DML-IV. For DFIV, we follow the original structure proposed in~\citet{Xu2020} with regularisers $\lambda 1$, $\lambda 2$ both set to 0.1 and weight decay of 0.001. For DeepIV, we use the same network architectures as action network and stage 2 network for DML-IV, with the dropout rate in~\citet{Hartford2017DeepPrediction} and weight decay of 0.001. For KIV, we use the Gaussian kernel, where the bandwidth is determined by the median trick as originally described by~\citet{Singh2019}, and we use the random Fourier feature trick with 100 dimensions.

\begin{table}[t]
    \caption{Network architecture for DML-IV and CE-DML-IV for the aeroplane ticket demand low-dimensional dataset. For the input layer, we provide the input variables. For mixture of Gaussians output, we report the number of components. The dropout rate is given in the main text.}
    \centering
    \subfloat[Action Network for $\hat{g}$]{
    \begin{tabular}{||c|c||}
    \hline
    \textbf{Layer Type} & \textbf{Configuration}  \\ [0.5ex]
    \hline \hline
    Input & $C,Z$\\      \hline
    FC + ReLU & in:3 out:128\\    \hline
    Dropout & - \\    \hline
    FC + ReLU & in:128 out:64\\    \hline
    Dropout & - \\    \hline
    FC + ReLU & in:64 out:32\\    \hline
    Dropout & - \\    \hline
    MixtureGaussian & 10\\    \hline

    \end{tabular}\label{tab:action_arch}}
    \hspace{30pt}
    \subfloat[Outcome Network for $\hat{s}$]{
    \begin{tabular}{||c|c||}
    \hline
    \textbf{Layer Type} & \textbf{Configuration}  \\ [0.5ex]
    \hline \hline
    Input & $C,Z$\\      \hline
    FC + ReLU & in:3 out:128\\    \hline
    Dropout & - \\    \hline
    FC + ReLU & in:128 out:64\\    \hline
    Dropout & - \\    \hline
    FC + ReLU & in:64 out:32\\    \hline
    Dropout & - \\    \hline
    FC & in:32 out:1\\    \hline
    \end{tabular}\label{tab:outcome_arch}}

    \vspace{10pt}
    \subfloat[Stage 2 Network for $\hat{h}$]{
    \begin{tabular}{||c|c||}
    \hline
    \textbf{Layer Type} & \textbf{Configuration}  \\ [0.5ex]
    \hline \hline
    Input & $C,A$\\      \hline
    FC + ReLU & in:3 out:128\\    \hline
    Dropout & - \\    \hline
    FC + ReLU & in:128 out:64\\    \hline
    Dropout & - \\    \hline
    FC + ReLU & in:64 out:32\\    \hline
    Dropout & - \\    \hline
    FC & in:32 out:1\\    \hline
    \end{tabular}\label{tab:stage2_arch}}

    \label{tab:demand_arch}
\end{table}

\subsection{Aeroplane Ticket Demand with MNIST}

For DML-IV and CE-DML-IV, we use a convolutional neural network (CNN) feature extractor, which we denote as \textit{ImageFeature}, described in~\cref{tab:mnist_arch}, for all networks. The full network architecture is described in~\cref{tab:mnist_demand_arch}; we use weight decay of 0.05. For DeepGMM, we use the same structure as the outcome network of DML-IV, with a dropout rate of 0.1 and weight decay of 0.05. For DFIV, we follow the original structure proposed in~\citet{Xu2020} with regularisers $\lambda 1$, $\lambda 2$ both set to 0.1 and weight decay of 0.05. For DeepIV, we use the same network architecture as the action network and stage 2 network for DML-IV, with the dropout rate in~\citet{Hartford2017DeepPrediction} and weight decay of 0.05. For KIV, we use the Gaussian kernel, where the bandwidth is determined by the median trick as originally described by~\citet{Singh2019}, and we use the random Fourier feature trick with 100 dimensions.

\begin{table}[t]
\caption{Network architecture of the feature extractor used for the aeroplane ticket demand dataset with MNIST. For each convolution
layer, we list the kernel size, input dimension and output dimension, where s stands for stride and p stands for padding. For max-pooling, we provide the
size of the kernel. The dropout rate here is set to 0.3. We denote this feature extractor as \textit{ImageFeature}.}
    \centering
    \begin{tabular}{||c|c||}
    \hline
    \textbf{Layer Type} & \textbf{Configuration}  \\ [0.5ex]
    \hline \hline
    Input & $28\times 28$\\      \hline
    Conv + ReLU & $3\times3\times32$, s:1, p:0\\    \hline
    Max Pooling & $2\times2$, s:2\\    \hline
        Dropout & -\\    \hline
    Conv + ReLU & $3\times3\times64$, s:1, p:0\\    \hline
    Max Pooling & $2\times2$, s:2\\    \hline
        Dropout & -\\    \hline
    Conv + ReLU & $3\times3\times64$, s:1, p:0\\    \hline
    Dropout & -\\    \hline
    FC + ReLU & in: 576, out:64\\    \hline
    \end{tabular}

\label{tab:mnist_arch}
\end{table}

\begin{table}[t]
    \caption{Network architecture for DML-IV and CE-DML-IV for the aeroplane ticket demand dataset with MNIST. For the input layer, we provide the input variables. For a mixture of Gaussians output, we report the number of components. The dropout rate is given in the main text.}
    \centering
    \subfloat[Action Network for $\hat{g}$]{
    \begin{tabular}{||c|c||}
    \hline
    \textbf{Layer Type} & \textbf{Configuration}  \\ [0.5ex]
    \hline \hline
    Input & ImageFeature$(C),Z$\\      \hline
    FC + ReLU & in:66 out:32\\    \hline
    Dropout & - \\    \hline
    MixtureGaussian & 10\\    \hline
    \end{tabular}}
    \hspace{30pt}
    \subfloat[Outcome Network for $\hat{s}$]{
    \begin{tabular}{||c|c||}
    \hline
    \textbf{Layer Type} & \textbf{Configuration}  \\ [0.5ex]
    \hline \hline
    Input & ImageFeature$(C),Z$\\      \hline
    FC + ReLU & in:66 out:32\\    \hline
    Dropout & - \\    \hline
    FC & in:32 out:1\\    \hline
    \end{tabular}}

    \vspace{10pt}
    \subfloat[Stage 2 Network for $\hat{h}$]{
    \begin{tabular}{||c|c||}
    \hline
    \textbf{Layer Type} & \textbf{Configuration}  \\ [0.5ex]
    \hline \hline
    Input & ImageFeature$(C),A$\\      \hline
    FC + ReLU & in:66 out:32\\    \hline
    Dropout & - \\    \hline
    FC & in:32 out:1\\    \hline
    \end{tabular}}
    \label{tab:mnist_demand_arch}
\end{table}

\subsection{IHDP and PM-CMR}
For the two real-world datasets, we use the same network architectures described in~\cref{tab:demand_arch} as in the low-dimensional ticket demand setting, where the input dimension is increased to 7 for all networks. We use a dropout rate of 0.1 and weight decay of 0.001. For DeepGMM, we use the same structure as the outcome network of DML-IV with dropout $=0.1$. For DFIV, we also use the same network architectures as in the low dimensional ticket demand setting with regularisers $\lambda 1$, $\lambda 2$ both set to 0.1 and weight decay of 0.001. For DeepIV, we use the same network architectures as the action network and stage 2 network of DML-IV, with a dropout rate of 0.1 and weight decay of 0.001. For KIV, we use the Gaussian kernel where the bandwidth is determined by the median trick as originally described by~\citet{Singh2019}, and we use the random Fourier feature trick with 100 dimensions.

\subsection{Valiadation and Hyper-Parameter Tuning}\label{appen:tune}

Validation procedures are crucial for tuning DNN hyper-parameters and optimizer parameters. All the DML-IV and CE-DML-IV training stages can be validated by simply evaluating the respective losses on held-out data, as discussed in~\citet{Hartford2017DeepPrediction}. This allows independent validation and hyperparameter tuning of the two first stage networks (the action and the outcome networks), and perform second stage validation using the best network selected in the first stage. This validation procedure guards against the ‘weak instruments’ bias~\cite{Bound1995} that can occur when the instruments are only weakly correlated with the actions variable (see detailed discussion in~\citet{Hartford2017DeepPrediction}).

\section{Additional Experimental Results}

In this section, we provide additional experimental results including the effects of weak IVs, performance with tree-based estimators, and a hyperparameter sensitivity analysis.

\subsection{Effects of Weak Instruments}\label{appen:weak_iv}

When the correlation between instruments and the endogenous variable (the action in our case) is weak, IV regression methods generally become unreliable~\cite{Andrews2019} because the weak correlation induces variance and bias in the first stage estimator thus induces bias in the second stage estimator, especially for non-linear IV regressions. In theory, DML-IV should be more resistant to biases in the first stage thanks to the DML framework, as long as the causal effect is identifiable under the weak instrument. Under this identifiability condition, Lemma 3.3, Theorem 3.4 and 3.5 all hold, and the convergence rate guarantees still apply. However, while causal identifiability with weak instruments are studied theoretically in the linear setting~\cite{Andrews2019}, such a theoretical study for non-linear IV models, to the best of our knowledge, does not exist due to the difficulty of analyzing non-linear models and estimators.

Experimentally, for the airplane ticket demand dataset, we alter the instrument strength by changing how much the instrument z affects the price p. Recall from~\cref{appen:demand} that $p=25+(z+3)\psi(t)+\omega$, where $\psi$ is a nonlinear function and $\omega$ is the noise. We add an IV strength parameter $\varrho$ such that $p=25+(\varrho\cdot z+3)\psi(t)+\omega$. In~\cref{tab:weak_iv}, we present the mean and standard deviation of the MSE of $\hat{h}$ for various IV strengths $\varrho$ from 0.01 to 1 and sample size $N=5000$. It is very interesting to see that DML-IV indeed performs significantly better than SOTA nonlinear IV regression methods under weak instruments.

\begin{table}[t]\setlength\extrarowheight{4pt}
\centering
\tiny
\begin{tabular}{||c|c|c|c|c|c|c||}
\hline
\textbf{IV Strength} & \textbf{1.0} & \textbf{0.8}                     & \textbf{0.6}                     & \textbf{0.4}                     & \textbf{0.2}                     & \textbf{0.01}                    \\\hline \hline
DML-IV              & \textbf{0.0676(0.0116)} & \textbf{0.0984(0.0161)} & \textbf{0.1295(0.0168)} & \textbf{0.1859(0.0376)} & \textbf{0.2899(0.0494)} & \textbf{0.4872(0.1295)} \\\hline
CE-DML-IV           & \textbf{0.0765(0.0119)} & \textbf{0.1064(0.0120)} & \textbf{0.1514(0.0203)} & \textbf{0.2070(0.0329)} & \textbf{0.3194(0.0572)} & \textbf{0.5302(0.1625)} \\\hline
DeepIV              & 0.1213(0.0209)          & 0.2039(0.0269)         & 0.3051(0.0415)          & 0.4476(0.0656)          & 0.6891(0.1210)          & 0.9293(0.2382)          \\\hline
DFIV                & 0.1124(0.0481)          & 0.1586(0.0320)          & 0.3080(0.1907)          & 0.8117(0.2779)          & 0.9622(0.3892)          & 1.6503(0.6845)          \\\hline
DeepGMM             & 0.2699(0.0522)          & 0.3330(0.1171)          & 0.4762(0.1056)          & 0.8666(0.2248)          & 1.0056(0.4334)          & 2.0218(0.6555)          \\\hline
KIV                 & 0.2312(0.0272)          & 0.3149(0.0218)          & 0.4275(0.0368)          & 0.6646(0.0538)          & 0.8099(0.0657)          & 1.226(0.1014)\\\hline
\end{tabular}
\caption{Results for the low-dimensional ticket demand dataset when the IV is weakly correlated with the action.}
\label{tab:weak_iv}
\end{table}

\subsection{Performance of DML-IV with tree-based estimators}\label{appen:tree-based}

The DML-IV framework allows for general estimators following the Neyman orthogonal score function. While deep learning is flexible and widely used in SOTA non-linear IV regression methods, Gradient Boosting and Random Forests regression are all good candidate estimators for DML-IV. In addition, as discussed in Lemma 3.3, the convergence rate and suboptimality guarantees in Theorem 3.4 and 3.5 both hold for these tree-based regressions.

Empirically, we replace the DNN estimators in DML-IV, CE-DML-IV and DeepIV with Random Forests and Gradient Boosting regressors (using scikit-learn implementation). DeepIV is a good baseline for comparison, since it optimizes directly using a non-Neyman-orthogonal score and allows for direct replacement of all DNN estimators with tree-based estimators. We use 500 trees for both regressors, with minimum samples required at each leaf node of 100 for the nuisance parameters and 10 for $\hat{h}$.

In~\cref{tab:tree_based}, we present the mean and standard deviation of the MSE of $\hat{h}$ with Random Forests and Gradient Boosting estimators on the aeroplane ticket demand dataset with various dataset sample sizes. The results demonstrate the benefits of our Neyman orthogonal score function, and interestingly the performance of Gradient Boosting is comparable to DNN estimators.

\begin{table}[ht]\setlength\extrarowheight{4pt}
\centering
\scriptsize
\begin{tabular}{||c|c|c|c|c||}
\hline
\textbf{IV Strength} & \textbf{Dataset Size} & \textbf{DNN (results in the paper)} & \textbf{Random Forests} & \textbf{Gradient Boosting} \\\hline \hline
DML-IV    & 2000                  & \textbf{0.1308(0.0206)}             & 0.1689(0.0172)          & \textbf{0.1301(0.0112)}    \\\hline
CE-DML-IV &    2000               & \textbf{0.1410(0.0246)}             & 0.1733(0.0198)          & \textbf{0.1329(0.0125)}    \\\hline
DeepIV    &   2000               & 0.2388(0.0438)                      & 0.2642(0.0261)          & 0.2052(0.0232)             \\\hline
DML-IV    & 5000                  & \textbf{0.0676(0.0129)}             & 0.1067(0.0131)          & \textbf{0.0632(0.0107)}    \\\hline
CE-DML-IV &   5000               & \textbf{0.0765(0.0119)}             & 0.1154(0.0138)          & \textbf{0.0699(0.0069)}    \\\hline
DeepIV    &    5000              & 0.1213(0.0209)                      & 0.1626(0.0128)          & 0.1020(0.0091)             \\\hline
DML-IV    & 10000                 & \textbf{0.0378(0.0094)}             & 0.0657(0.0062)          & \textbf{0.0482(0.0079)}    \\\hline
CE-DML-IV &  10000                & \textbf{0.0442(0.0070)}             & 0.0721(0.0039)          & \textbf{0.0523(0.0059)}    \\\hline
DeepIV    &  10000             & 0.0714(0.0140)                      & 0.1106(0.0080)          & 0.1017(0.0075)\\\hline
\end{tabular}
\caption{Results for the low-dimensional ticket demand dataset using tree-based estimators compared to DNN estimators.}
\label{tab:tree_based}
\end{table}

\subsection{Sensitivity analysis for different Hyperparameters}\label{appen:sensitivity}

The tunable hyperparameters in DML-IV are the learning rate, network width, weight decay and dropout rate (see~\cref{appen:networks}). As a sensitivity analysis, we provide results for the mean and standard deviation of the MSE of the DML-IV estimator $\hat{h}$ with different hyperparameter values for both the low-dimensional and high-dimensional datasets with sample size N=5000 in~\cref{tab:ablation_low} and~\cref{tab:ablation_high}. Overall, we see that DML-IV is not very sensitive to small changes of the hyperparameters.

\begin{table}[t]\setlength\extrarowheight{4pt}
    \centering
    \scriptsize
    \begin{tabular}{||c|c|c|c|c|c||}
    \hline
     \textbf{Learning Rate} & \textbf{Weight Decay}& \textbf{Dropout}& \textbf{DNN Width} & \textbf{DML-IV} &\textbf{CE-DML-IV} \\\hline \hline
0.0002 & 0.001  & 0.1  & 128 & \textbf{0.0676(0.0129)} & \textbf{0.0765(0.0119)} \\\hline
0.0005 &        &      &     & 0.0752(0.0122)          & 0.0897(0.0196)          \\\hline
0.0001 &        &      &     & \textbf{0.0703(0.0195)} & \textbf{0.0794(0.0201)} \\\hline
       & 0.0005 &      &     & 0.0794(0.0185)          & 0.0823(0.0149)          \\\hline
       & 0.005  &      &     & 0.0765(0.0135)          & 0.0809(0.0159)          \\\hline
       & 0.01   &      &     & 0.0820(0.0162)          & 0.0865(0.0174)          \\\hline
       &        & 0.05 &     & \textbf{0.0715(0.0074)} & \textbf{0.0813(0.0089)} \\\hline
       &        & 0.2  &     & 0.0836(0.0100)          & 0.0919(0.0157)          \\\hline
       &        &      & 64  & 0.0830(0.0162)          & 0.0924(0.0121)          \\\hline
       &        &      & 256 & 0.0943(0.0179)          & 0.0981(0.0126)          \\\hline
       & 0.0005 & 0.2  &     & 0.0805(0.0133)          & 0.0910(0.0106)          \\\hline
       & 0.005  & 0.05 &     & \textbf{0.0672(0.0116)} & \textbf{0.0742(0.0102)} \\\hline
       & 0.01   & 0.05 &     & 0.0825(0.0152)          & 0.0914(0.0125)          \\\hline
       &        & 0.2  & 256 & 0.0810(0.0129)          & 0.0852(0.0121)          \\\hline
       &        & 0.05 & 64  & 0.0907(0.0149)          & 0.0963(0.0161)          \\\hline
       & 0.005  &      & 256 & 0.0939(0.0146)          & 0.0991(0.0093)\\\hline
    \end{tabular}
    \caption{Results for the low-dimensional ticket demand dataset for a range of hyperparameter values. The default hyperparameters in this case are: learning rate=0.0002, weight decay=0.001, dropout=0.1 and DNN width 128.}
    \label{tab:ablation_low}
\end{table}
\begin{table}[ht]\setlength\extrarowheight{4pt}
    \centering
     \scriptsize
    \begin{tabular}{||c|c|c|c|c|c||}
    \hline
\textbf{Learning Rate} & \textbf{Weight Decay} & \textbf{Dropout} & \textbf{CNN Channels} & \textbf{DML-IV}          & \textbf{CE-DML-IV}      \\\hline \hline
0.001                  & 0.05                  & 0.2              & 64                    & \textbf{0.3513(0.0125)}  & \textbf{0.3808(0.0150)} \\\hline
0.0005                 &                       &                  &                       & 0.4063(0.0129)           & 0.5008(0.0369)          \\\hline
0.002                  &                       &                  &                       & 0.3659(0.0219)           & 0.4133(0.0267)          \\\hline
0.005                  &                       &                  &                       & \textbf{0.3377(0.0218)}  & \textbf{0.3555(0.0202)} \\\hline
& 0.01                  &                  &                       & 0.3935(0.0176)           & 0.4461(0.0478)          \\\hline
& 0.02                  &                  &                       & \textbf{0.3595(0.03013)} & \textbf{0.3851(0.0293)} \\\hline
& 0.1                   &                  &                       & 0.4066(0.0172)           & 0.5160(0.0329)          \\\hline
&                       & 0.1              &                       & 0.4136(0.0211)           & 0.5386(0.0398)          \\\hline
&                       & 0.3              &                       & 0.3857(0.0171)           & 0.4002(0.0249)          \\\hline
&                       &                  & 128                   & 0.4176(0.01941)          & 0.5129(0.0630)          \\\hline
&                       &                  & 256                   & 0.4942(0.0226)           & 0.6180(0.0396)          \\\hline
& 0.1                   & 0.1              &                       & 0.4163(0.0214)           & 0.5952(0.0343)          \\\hline
& 0.01                  & 0.3              &                       & 0.3636(0.0186)           & 0.3995(0.0250)          \\\hline
&                       & 0.3              & 128                   & 0.4006(0.0187)           & 0.4764(0.0216)          \\\hline
&                       & 0.3              & 256                   & \textbf{0.3429(0.0215)}  & \textbf{0.3971(0.0264)} \\\hline
& 0.1                   &                  & 256                   & 0.4170(0.0283)           & 0.5335(0.0371)\\\hline
                       
\end{tabular}
\caption{Results for the high-dimensional ticket demand dataset for a range of hyperparameter values. The default hyperparameters in this case are: learning rate 0.001, weight decay=0.05, dropout=0.2 and 64 CNN channels.}
\label{tab:ablation_high}
\end{table}

\end{document}